\documentclass[10pt,journal,compsoc]{IEEEtran}
\usepackage{amsmath,amsfonts,amssymb}
\usepackage{algorithmic}
\usepackage{algorithm}
\usepackage{array}

\usepackage{textcomp}
\usepackage{stfloats}
\usepackage{url}
\usepackage{verbatim}
\usepackage{graphicx}
\usepackage{cite}
\usepackage{booktabs}

\newcommand{\btheta}{\boldsymbol{\theta}}

\newcommand{\bSigma}{\boldsymbol{\Sigma}}

\newcommand{\bomega}{\boldsymbol{\omega}}
\usepackage[hidelinks]{hyperref}

\newcommand{\calT}{\mathcal{T}}
\newcommand{\calD}{\mathcal{D}}
\newcommand{\calL}{\mathcal{L}}

\newcommand{\A}{\mathbf{A}}
\newcommand{\B}{\mathbf{B}}
\newcommand{\boldS}{\mathbf{S}}
\newcommand{\R}{\mathbf{R}}
\newcommand{\M}{\mathbf{M}}
\newcommand{\W}{\mathbf{W}}
\newcommand{\U}{\mathbf{U}}
\newcommand{\V}{\mathbf{V}}

\newcommand{\bH}{\mathbf{H}}
\newcommand{\bg}{\mathbf{g}}

\usepackage{algorithm}
\usepackage{algorithmic}

\DeclareMathOperator*{\argmin}{arg\,min}

\usepackage[table]{xcolor}
\usepackage{bbm}
\usepackage{amsthm}
\usepackage{caption}
\usepackage{subcaption}
\usepackage{makecell}
\usepackage{multirow}
\newtheorem{theorem}{Theorem}
\newtheorem{theoremnew}{Theorem}

\begin{document}

\title{Hessian Aware Low-Rank Perturbation for Order-Robust Continual Learning}

\author{Jiaqi Li, Yuanhao Lai, Rui Wang, Changjian Shui, Sabyasachi Sahoo, Charles X. Ling, Shichun Yang, Boyu Wang, Christian Gagn\'{e}, Fan Zhou

\IEEEcompsocitemizethanks{
    \IEEEcompsocthanksitem Corresponding to Fan Zhou (E-mail: fanzhou@buaa.edu.cn).
    \IEEEcompsocthanksitem Jiaqi Li, Yuanhao Lai, Charles X. Ling, and Boyu Wang are with the Dept. of Computer Science, University of Western Ontario, London, ON, Canada. E-mail: jli3779@uwo.ca. Boyu Wang is also with the Vector Institute, Toronto, ON, Canada.
    \IEEEcompsocthanksitem Rui Wang, Shichun Yang, and Fan Zhou are with Dept. of Automotive Engineering, Beihang University, Beijing, China.
    \IEEEcompsocthanksitem Changjian Shui is with Dept. of Electrical and Computer Engineering, Mcgill University, Montreal, QC, Canada.
    \IEEEcompsocthanksitem Sabyasachi Sahoo and Christian Gagn\'{e} are with the Dept. of Electrical Engineering and Computer Engineering, Laval University, Qu\'{e}bec City, QC, Canada. Christian Gagn\'{e} is also with Mila, Quebec, Canada.
}
\thanks{Manuscript received 15 August 2023; revised 5 June 2024; accepted 22 June 2024.}
}

\markboth{Journal of \LaTeX\ Class Files,~Vol.~xx, No.~x, xxx~2023}%
{Li \MakeLowercase{{et al.}}: A Sample Article Using IEEEtran.cls for IEEE Journals}

\IEEEpubid{This work has been submitted to the IEEE for possible publication. Copyright may be transferred without notice, after which this version may no longer be accessible.}


\IEEEtitleabstractindextext{
\begin{abstract}
Continual learning aims to learn a series of tasks sequentially without forgetting the knowledge acquired from the previous ones. 
In this work, we propose the Hessian Aware Low-Rank Perturbation algorithm for continual learning. By modeling the parameter transitions along the sequential tasks with the weight matrix transformation, we propose to apply the low-rank approximation on the task-adaptive parameters in each layer of the neural networks. Specifically, we theoretically demonstrate the quantitative relationship between the Hessian and the proposed low-rank approximation. The approximation ranks are then globally determined according to the marginal change of the empirical loss estimated by the layer-specific gradient and low-rank approximation error. Furthermore, we control the model capacity by pruning less important parameters to diminish the parameter growth. We conduct extensive experiments on various benchmarks, including a dataset with large-scale tasks, and compare our method against some recent state-of-the-art methods to demonstrate the effectiveness and scalability of our proposed method. Empirical results show that our method performs better on different benchmarks, especially in achieving task order robustness and handling the forgetting issue. The source code is at \href{https://github.com/lijiaqi/HALRP}{https://github.com/lijiaqi/HALRP}.
\end{abstract}

\begin{IEEEkeywords}
Non-stationary Environment, Continual Learning, Low-rank Perturbation, Hessian information
\end{IEEEkeywords}

\IEEEdisplaynontitleabstractindextext

\IEEEpeerreviewmaketitle
}
\maketitle
   
\section{Introduction}
\label{Sect:Intro}
\IEEEPARstart{T}{he} conventional machine learning paradigm assumes all the data are simultaneously accessed and trained. However, in practical scenarios, data are often collected from different tasks and sequentially accessed in a specific order.
Continual Learning (CL) aims to gradually learn from novel tasks and preserve valuable knowledge from previous ones.
Despite this promising paradigm, CL is usually faced with a dilemma between \emph{memory stability} and \emph{learning plasticity}: when adapting among the dynamic data distributions, it has been shown that the neural networks can easily forget the learned knowledge of previous tasks when facing a new one, which is known as \emph{Catastrophic Forgetting} (CF)~\cite{de2021continual}. 

One possible reason for this forgetting issue can be the parameter drift from the previous tasks to the new ones, which is caused by the optimization process with the use of stochastic gradient descent and its variants~\cite{saha2021gradient,farajtabar2020orthogonal}. To mitigate the obliviousness to past knowledge, some works~\cite{kirkpatrick2017overcoming,jung2020continual,titsias2019functional,wang2021training,kong2023Overcoming} proposed to constrain the optimization objective for the knowledge of new tasks with additional penalty regularization terms. These approaches were shown as ineffective under the scenarios of a large number of tasks, lacking long-term memory stability in real-world applications.

To address this problem, some methods choose to expand the model during the dynamic learning process in a specific way, leading to inferencing with task-specific parameters for each task. For example, a recent work~\cite{APD:YoonKYH20} proposed to decompose the model as task-private and task-shared parameters via an additive model parameters decomposition. However, this method only applies an attention vector for the masking of the task-shared parameters. Furthermore, the private parameters were controlled by a regularization and consolidation process, which lead to the linear model increment along with the increasing task number, damaging the \emph{scalability} for deploying a CL system. Another solution to reduce the parameter increase is to apply factorization for the model parameters. In this regard, \cite{mehta2021continual} proposed a low-rank factorization method for the model parameters decomposition with a Bayesian process inference. However, this approach required a large rank for achieving desirable accuracy and also suffered from ineffectiveness in some complex data scenarios.  

In this work, we propose a low-rank perturbation method to learn the relationship between the learned model (base parameters) and the parameters for new tasks. By formulating a flexible weight transition process, the model parameters during the CL scenario are decomposed as task-shared ones and task-adaptive ones. The former can be adopted from the base task to new tasks. The latter conforms to a low-rank matrix, and the number of relevant parameters can be effectively reduced for every single layer in a neural network across the tasks. With a simple warm-up training strategy, low-rank task adaptive parameters can be efficiently initialized with \emph{singular value decomposition}.

Furthermore, to determine which ranks should be preserved for the low-rank approximation in different layers in a model, we propose to measure the influences of the introduced low-rank parameters by the \emph{Hessian-aware risk perturbation} across layers of the whole model. This allows the model to automatically assign larger ranks to a layer that contributes more to the final performance under a specific parameter size budget. We theoretically support this approach by providing a formal demonstration that the empirical losses derived from the proposed low-rank perturbation are bounded by the Hessian and the approximation rate of the decomposition. This leads to a Hessian-aware framework that enables the model to determine perturbation ranks that minimize the overall empirical error. Lastly, we apply a pruning technique to control the introduced parameter size and reduce some less important perturbation parameters through the aforementioned Hessian-aware framework by evaluating their importance.

Based on the theoretical analysis, we proposed the \textbf{H}essian \textbf{A}ware \textbf{L}ow-\textbf{R}ank \textbf{P}erturbation (HALRP) algorithm, which enables the model to leverage the Hessian information to automatically apply the low-rank perturbation according to a certain approximation rate. To summarize, the contributions of our work mainly lie in three-fold: 
\begin{itemize}
    \item We proposed a Hessian-aware low-rank perturbations framework that allows efficient memory and computation requests for learning continual tasks.
    \item Theoretical analyses were investigated to show that Hessian information can be used to quantitatively measure the influences of low-rank perturbation on the empirical risk. This leads to an automatic rank selection process to control the model increment.
    \item Extensive experiments on several benchmarks and recent state-of-the-art baselines (\emph{e.g.}, regularization-, expansion-, and replay-based) were conducted to demonstrate the effectiveness of our method. The results testify the superiority of our HALRP, in terms of accuracy, computational efficiency, scalability, and task-order robustness.
\end{itemize}

\section{Related Work}
\textbf{Regularization-based approaches} diminish catastrophic forgetting by penalizing the parameter drift from the previous tasks using different regularizers~\cite{kirkpatrick2017overcoming,jung2020continual,titsias2019functional}.
Batch Ensemble~\cite{wen2020batchensemble} designed an ensemble weight generation method by the Hadamard product between a shared weight among all ensemble members and an ensemble member-specific rank-one matrix. 
\cite{chaudhry2020continual} indicated that learning tasks in different low-rank vector subspaces orthogonal to each other can minimize task interferences. 
\cite{Yang2023Learning} applied a regularizer with decoupled prototype-based loss, which can improve the intra-class and inter-class structure significantly.
Compared to this kind of approach, our method applied task-specific parameters and introduced an explicit weight transition process to leverage the knowledge from the previous tasks and overcome the forgetting issue.

\textbf{Expansion-based methods} utilize different subsets of model parameters for each task.
\cite{rusu2016progressive} proposed to memorize the learned knowledge by freezing the base model and progressively expanding the new sub-model for new tasks. 
In \cite{yoon2017lifelong}, the reuse and expansion of networks were achieved dynamically by selective retraining, with splitting or duplicating components for newly coming tasks. Some work (\emph{e.g.},~\cite{li2019learn}) tried to determine the optimal network growth by neural architecture search. To balance memory stability and learning plasticity, \cite{Liu2023Balanced} adopted a distillation-based method under the class incremental scenario. To achieve the scalability and the robustness of task orders, Additive Parameter Decomposition (APD)~\cite{APD:YoonKYH20} adopted sparse task-specific parameters for novel tasks in addition to the dense task-shared ones and performed hierarchical consolidation within similar task groups for the further knowledge sharing.
~\cite{ge2023clr} introduced Channel-Wise Linear Reprogramming (CLR) transformations on the output of each convolutional layer of the base model as the task-private parameters. However, this method relied on the prior knowledge from a disjoint dataset (e.g., ImageNet-1K).
Winning Subnetworks (WSN)~\cite{kang2022forget} jointly learned the shared network and binary masks for each task. To address the forgetting issue, only the model weights that had not been selected in the previous tasks were tendentiously updated during training. But this method still needs to store task-specific masks for the inference stage and relies on extra compression processes for mask encoding to achieve scalability.

\textbf{Replay-/Memory-based approaches} usually leverage different types of replay buffers (\emph{e.g.}, 
\cite{riemer2018learning,rebuffi2017icarl,chaudhry2020continual,zhang2022memory,chen2023Multi,sun2021continual,ho2023prototype}) to memorize a small episode of the previous tasks, which can be rehearsed when learning the novel ones to avoid forgetting.
In~\cite{ho2023prototype}, a dynamic prototype-guided memory replay module was incorporated with an online meta-learning framework to reduce memory occupation.
\cite{wang2021triple} proposed to process both specific and generalized information by the interplay of three memory networks.
To avoid the repeated inferences on previous tasks, Gradient Episodic Memory~\cite{lopez2017gradient} and its variant~\cite{chaudhry2019efficient} projected the new gradients into a feasible region that is determined by the gradients on previous task samples. Instead, Gradient Projection Memory (GPM)~\cite{saha2021gradient} chose to directly store the bases of previous gradient spaces to guide the direction of parameters update on new tasks. Compared to these approaches, our proposed method does not rely on extra storage for the data or gradient information of previous tasks, avoiding privacy leakage under some sensitive scenarios.

\textbf{Low-rank Factorization for CL}
 Low-rank factorization~\cite{TaiXWE15_LowRank} has been widely studied in deep learning to decompose the parameters for model compression~\cite{idelbayev2020low,phan2020stable} or data projection~\cite{wang2012LowRankKernel,zhu2019LowRankSparse}. In the context of CL,~\cite{mehta2021continual} considers the low-rank model factorization and automatic rank selection per task for variational inference, which requires significant large rank increments per task to achieve high accuracy. GPM~\cite{saha2021gradient} applied the singular value decomposition on the representations and store them in the memory. \cite{chaudhry2020continual} proposed to learn tasks by low-rank vector sub-spaces to avoid a joint vector space that may lead to interferences among tasks.
 The most similar work is Incremental Rank Updates (IRU)~\cite{hyder2022incremental}. Compared to the decomposition in IRU, we adopted the low-rank approximation on the residual representation in the weights transitions. Furthermore, the rank selection in our work was dynamically and automatically determined according to Hessian-aware perturbations, rather than the manual rank increment in~\cite{hyder2022incremental}.

\begin{table}[t]
  \renewcommand\arraystretch{1.1}
\caption{List of main notations.}
\label{tab:notations}
\centering
\begin{tabular}{ll}
\toprule
Notation & Description \\ \hline
$T$,$t$ & Total task number, task index\\ 
$\mathcal{T}_t, \mathcal{D}_t$ & $t$-th task, and the related dataset\\ 
$\mathbf{x}, y$ & input/sample, class label\\ 
$\W$ & model weights\\
$\R,\boldS,\B$ & task private model weights\\
$\U,\bSigma,\V$ & SVD decomposition on $\B$ shown in Sect.~\ref{Sect:svd-background}\\
$\mathrm{diag}(\bSigma)$ & the set of diagonal elements of $\bSigma$\\
$\mathrm{Diag}(\sigma_1,\cdots,\sigma_n)$ & diagonal matrix elements $\{\sigma_1,\cdots,\sigma_n\}$\\
$\M^{(k)}$ & $k$-rank low-rank approximation of $\M$ \\
$\|\cdot\|_F$ & Frobenius norm\\
$\mathcal{L}$ & loss function for the task (e.g, cross-entropy)\\
$\mathcal{L}_{\text{reg}}$ & regularization loss\\
$\bH$ & Hessian matrix\\
$\alpha$ & loss approximation rate for rank selection\\
$l$ & layer index in a neural network\\ 
$\bg_l$ & gradient vector for layer $l$\\
$r_l$ & full rank of layer $l$\\
$k_l$ & approximated rank for layer $l$\\
$n$ & number of total training epochs\\
$n_r$ & number of warm-up epochs\\
\bottomrule
\end{tabular}
\end{table}

\section{Preliminary and Background}
This section introduces the basic knowledge of continual learning and singular value decomposition. 
In Table~\ref{tab:notations}, we provide a list of the main notations used in our paper. Specifically, the bold uppercase (or lowercase) letters indicate the matrices (or vectors) in the remaining part.

\subsection{Continual Learning}
Assume the learner receives a series of $T$ tasks $\{\mathcal{T}_0,\dots,\mathcal{T}_{T-1}\}$ sequentially, and denote the dataset of the $t^{th}$ task as $\mathcal{D}_t = \{\mathbf{x}^i_t, y^i_t\}_{i=1}^{N_t}$, where $\mathbf{x}^i_t$ and $y^i_t$ corresponds to the $i^{th}$ instance and label in the total of $N_t$ data points. In the context of CL, the dataset $\mathcal{D}_t$ will become inaccessible after the time step $t$.

\subsection{Low-Rank Approximation of Matrix with SVD}
\label{Sect:svd-background}
The singular value decomposition (SVD) factorizes a rectangular matrix $\B\in \mathbb{R}^{J\times I}$ with three matrices\footnote{In this paper, we assume the eigenvalues are always sorted with descending order in such SVD decompositions, i.e., $\sigma_1\geq\sigma_2\geq...\geq\sigma_r$.}: $\B = \U \bSigma \V^\top$, where $\U \in \mathbb{R}^{J\times J}$, $\V \in \mathbb{R}^{I\times I}$, and $\bSigma \in \mathbb{R}^{J\times I}$.
Denote the rank of $\B$ with $r\leq \min\{I,J\}$, then $\B$ can be further expressed as $\B = \sum_{i=1}^{r} \sigma_i \mathbf{u}_i \mathbf{v}_i^\top$, where $\sigma_i\in \mathrm{diag}(\bSigma)$ is the singular values, $\mathbf{u}_i$ and $\mathbf{v}_i$ are respectively the left and right singular vectors. 
In this work, we take the property of $k$-rank approximation of $\B$ (with $k\leq r$) by leveraging the \emph{Eckart–Young–Mirsky theorem}~\cite{eckart1936approximation}, which can be expressed with \emph{top-$k$ leading} singular values and vectors: $\B^{(k)} = \U^{(k)} \bSigma^{(k)} (\V^{(k)})^{\top} = \sum_{i=1}^k \sigma_i \mathbf{u}_i \mathbf{v}_i^\top$, where $\U^{(k)}$, $\bSigma^{(k)}$, $\V^{(k)}$ are the corresponding \emph{leading} principal sub-matrices\footnote{With a slight abuse of notation, the notation $\M^{(k)}$ refers to the $k$-rank approximation of the matrix $\M$ according to the context: (1) $\U^{(k)},\bSigma^{(k)}, \V^{(k)}$ indicate the top-$k$ leading submatrices; (2) $\B^{(k)}$ means the $r$-rank low-rank approximation for $\B$, then same for $\W^{(k)}$.}.
$k$ is chosen to tolerate a certain approximation error under the Frobenius norm $\|\cdot\|_{F}$ (See Appendix A for the proof):  
\begin{equation}
\begin{split}
    \| \B - \B^{(k)} \|_{F} =\sqrt{\sigma_{k+1}^2 + \dots + \sigma_r^2}.
\end{split}
\label{Eq.approxKrank}
\end{equation}

\section{Methodology}
\label{sec.Methodology}
 \begin{figure}[t]
    \centering
    \includegraphics[width=\linewidth]{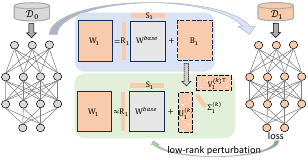} 
    \caption{Low rank decomposition between $\calT_1$ and $\calT_0$}
    \label{fig:LowRankDecomposition}
\end{figure}
Our framework leverages the low-rank approximation of neural network weights. In the following parts, we present the methodology for handling the fully connected layers and the convolutional layers. The analysis herein can be applied to any layer of the model. Without loss of generality, we omit the layer index $l$ in this section. For simplicity, we illustrate the learning process using tasks $\mathcal{T}_0$ and $\mathcal{T}_1$ in Sections~\ref{Sect:DecomposeFC} and \ref{Sect:DecomposeConv}, which can be applied to the successive new tasks as shown in Section~\ref{Sect.FullProcess}. 

\subsection{Fully Connected Layers}
\label{Sect:DecomposeFC}
We first consider linear layers of neural networks. We begin by learning task $\mathcal{T}_0$ without any constraints on the model parameters. Specifically, we train the model by minimizing the empirical risk to get the base weights 
$\W^{\text{base}} = \argmin_{\W} \calL(\W; \mathcal{D}_0)$,
where $\W^{\text{base}} \in \mathbb{R}^{J\times I}$, $J$ is the output dimension, and $I$ is the input dimension of the layer.

Then, when the task $\mathcal{T}_1$ comes to the learner, we can train the model fully on $\calD_1$ and get the updated weights $\W_1 \in \mathbb{R}^{J\times I}$. 
However, undesired model drifts can lead to worse performance on $\calT_0$ if no constraints are applied to the parameters update. 
Previous work~\cite{kirkpatrick2017overcoming} applied a $L_2$ regularization to enforce that the weights learned on the new model will not be too far from those of the previous tasks.
Although this kind of method only expands the base model with limited parameters, it can lead to worse performance on both $\calT_0$ and $\calT_1$. 

To pursue a better trade-off between the model size increment and overall performance on both $\calT_0$ and $\calT_1$,
we assume the unconstrained trained parameters $\W_1$ on $\mathcal{T}_1$ can be transformed from $\W^{\text{base}}$ by the \emph{low-rank weight perturbation} (LRWP, as illustrated in Figure~\ref{fig:LowRankDecomposition}),
\begin{equation}
    \W_{1} = \R_{1} \W^{\text{base}} \boldS_{1} + \B_{1}
    \label{Eq.LowRankDecoposition}
\end{equation}
where $\B_{1}\in \mathbb{R}^{J\times I}$ is 
a sparse low-rank matrix, $\R_{1}=\textrm{Diag}(r_{1},\dots, r_{J})$,
and $\boldS_{1} = \textrm{Diag}(s_{1},\dots, s_{I})$ are task-private parameters for $\calT_1$\footnote{We will show later that $\B$ can be further approximated by $\U^{(k)},\bSigma^{(k)},\V^{(k)}$ with low-rank approximation, so we can write the task-private weights for task $t$ as either $\{\R_t,\boldS_t,\B_t^{(k)}\}$ or $\{\R_t,\boldS_t,\U_t^{(k)},\bSigma_t^{(k)},\V_t^{(k)}\}$.}. 

The form of the proposed low-rank decomposition in Eq.~\ref{Eq.LowRankDecoposition} has differences with the additive parameter decomposition proposed in~\cite{APD:YoonKYH20}.
In fact, the task-adaptive bias term $\B_{1}$ conforms to a low-rank matrix, which reduces both the parameter storage and the computational overhead. 
Moreover, the task-adaptive mask terms $\R_{1}$ and $\boldS_{1}$ include both row-wise and column-wise scaling parameters instead of only column-wise parameters like \cite{APD:YoonKYH20} to allow smaller and possibly sparser discrepancy:
$\B_{1} = \W_{1}-\R_{1} \W^{\text{base}} \boldS_{1}.$

In addition, regarding the parameter estimation, 
we can substitute $\B_1$ with its $k$-rank approximation
$\B_1^{(k)} = \sum_{i=1}^{k} \sigma_i \mathbf{u}_i \mathbf{v}_i^\top$ 
in Eq.~\ref{Eq.LowRankDecoposition}
and directly minimize the empirical risk on task $\mathcal{T}_1$ to
get the parameter estimation through stochastic gradient descent with the random initialization of $\{\mathbf{u}_i\}_{i=1}^k$ and $\{\mathbf{v}_i\}_{i=1}^k$.
However, in practice, we found it challenging to learn a useful decomposition in this way (\emph{i.e.}, 
the estimations converge to a local optimum quickly) and 
the empirical risk minimization does not benefit from the low-rank decomposition.
Due to the homogeneity among the one-rank components $\{\mathbf{u}_i \mathbf{v}_i^{\top}\}_{i=1}^k$, random initialization cannot sufficiently distinguish them and hence make the gradient descent ineffective.

To address the aforementioned ineffective training problem,
we first train the model fully on task $\mathcal{T}_1$ for a few epochs (\emph{e.g.}, one or two) to get rough parameter estimations  $\W_1^{\text{free}}\in \mathbb{R}^{J\times I}$ of the new task, indicating free-trained weights without any constraints. Then, a good warm-up initialized values for Eq.~\ref{Eq.LowRankDecoposition}
can then be obtained by solving
the following least squared error (LSE) minimization objective,
\begin{equation}
\argmin_{\R, \boldS, \B}  \|\W_1^{\text{free}}-  \R \W^{\text{base}} \boldS - \B\|_{F}^2
\label{Eq.Minimize_warmupGoal}
\end{equation}
As $\R, \boldS, \B$ can correlate, minimizing them simultaneously can be difficult. Thus, we minimize Eq.~\ref{Eq.Minimize_warmupGoal} alternately. First, fix $\boldS, \B$ and solve $\R$, and so on.
{
\begin{equation}
\begin{split}
    &\R_1^{\text{free}} = \argmin_{\R} \|\W_1^{\text{free}} -  \R \W^{\text{base}}\|_{F}^2\\
    &\boldS_1^{\text{free}} = \argmin_{\boldS}  \|\W_1^{\text{free}}  -  \R_1^{\text{free}} \W^{\text{base}} \boldS \|_{F}^2\\
    &\B_1^{\text{free}} = \W_1^{\text{free}} - \R_1^{\text{free}} \W^{\text{base}} \boldS_1^{\text{free}}
\end{split}
\label{Eq.Solve_RSB}
\end{equation}
}

Eq.~\ref{Eq.Solve_RSB} can be solved efficiently via LSE minimization. For example, for {\small$\R$}, we have {\small$\|\W^{\text{free}} -  \R \W^{\text{base}}\|_{F}^2 = \sum_{i=1}^I \sum_{j=1}^J (w^{\text{free}}_{ji} -r_jw^{\text{base}}_{ji})^2$}, where $w^{\text{free}}_{ji}$ and $w^{\text{base}}_{ji}$ are the elements of $j$-th row and $i$-th column of {\small$\W^{\text{free}}$} and {\small$\W^{\text{base}}$}, respectively. Applying the derivative \emph{w.r.t.} $r_j$ and let it equal 0, we have $r_j = (\sum_{i} w_{ji}^{\text{free}}w_{ji}^{\text{base}})/\sum_i (w_{ji}^{\text{base}})^2$. 
Then, we can solve {\small$\boldS$} and {\small$\B$} (see Appendix C for more details).
Furthermore, with the SVD solver denoted by $\texttt{SVD} (\cdot)$,
the values of the low-rank decomposition for {\small$\B^{\text{free}}$} can be obtained:
\begin{equation}
\label{Eq.SVD_FC}
    \begin{split}
         \U_1^{\text{free}}, \bSigma_1^{\text{free}}, \V_1^{\text{free}} &\leftarrow \texttt{SVD}(\B_1^{\text{free}})
    \end{split}
\end{equation}
A $k$-rank approximation $\U_{1}^{(k) {\text{free}}}$, 
$\bSigma_{1}^{(k){\text{free}}}$, $\V_{1}^{(k) {\text{free}}}$ 
can be obtained by retaining their corresponding leading principal submatrix of order $k\leq r$. 
So we can obtain a $k$-rank approximation\footnote{With this approximation, the extra parameters introduced for each successive task will be $\mathcal{O}((I+J)(k+1)+k)$. Then the model increment ratio is $\rho=\frac{\text{size}\{\R,\boldS,\U^{(k)},\bSigma^{(k)},\V^{(k)}\}}{\text{size}\{\W^{\text{base}}\}}=\frac{(I+J)(k+1)+k}{IJ}\ll 1$ in practice.} to $\W_1^{\text{free}}$ by
\begin{equation}
\label{Eq.lowrankWfree}
\W_1^{{\text{free}}} \approx \W_1^{(k){\text{free}}} = \R_1^{\text{free}} \W^{\text{base}} \boldS_1^{\text{free}} +
\B_1^{(k){\text{free}}},
\end{equation}
where {\small$\B_1^{(k){\text{free}}} = \U_1^{(k){\text{free}}} \bSigma_1^{(k){\text{free}}} (\V_1^{(k){\text{free}}})^\top $}.

Finally, we can initialize the values in Eq.~\ref{Eq.LowRankDecoposition}
with {\small$\R_1^{\text{free}}$, $\boldS_1^{\text{free}}$, $\U_{1}^{(k){\text{free}}}$, $\bSigma_{1}^{(k){\text{free}}}$} and {\small$\V_{1}^{(k){\text{free}}}$},
and then fine-tune their estimates by minimizing the empirical risk on task $\mathcal{T}_1$ to achieve better performance.
The proposed training technique not only enables 
well-behaved estimations of the low-rank components but also 
sheds light on how to select approximation ranks of different layers to achieve an optimal trade-off between model performance and parameter size, as discussed in Section~\ref{sec.SelectRank}.

\subsection{Convolutional Layers}
\label{Sect:DecomposeConv}
In addition, we can decompose the weights of a convolutional layer in a similar way.
Suppose that the size of the convolutional kernel is $d\times d$. The base weights of the convolution layer for task $\mathcal{T}_0$ would be a tensor $\W_{\text{conv}}^{\text{base}} \in\mathbb{R}^{d\times d\times J \times I}$. 
Similar to Eq.~\ref{Eq.LowRankDecoposition}, 
a low-rank weight perturbation for transforming
$\W_{\text{conv}}^{\text{base}}$ to $\W_{\text{conv},1}\in\mathbb{R}^{d\times d\times J \times I}$ is,
\begin{equation}
    \W_{\text{conv},1} =  \R_{\text{conv},1} \otimes \W_{\text{conv}}^{\text{base}} \otimes \boldS_{\text{conv},1} \oplus \B_{\text{conv},1}
    \label{Eq.LowRankDecopositionCNN}
\end{equation}
where $\R_{\text{conv},1}\in\mathbb{R}^{1\times 1\times J \times 1}$,
$\boldS_{\text{conv},1}\in\mathbb{R}^{1\times 1\times 1 \times I}$ and 
$\B_{\text{conv},1}\in\mathbb{R}^{1\times 1\times J \times I}$ is sparse low-rank tensor (matrix), 
$\otimes$ and $\oplus$ are element-wise tensor multiplication and
summation operators that will automatically expand tensors to be of equal sizes, following the broadcasting semantics of some popular scientific computation package like Numpy~\cite{harris2020array} or PyTorch~\cite{paszke2017automatic}.
Thus, the number of parameters added through this kind of decomposition is still $\mathcal{O}(I+J)$.

To get the estimations of the introduced parameters,
we can solve a similar LSE problem as in Eq.~\ref{Eq.Minimize_warmupGoal} to obtain initial estimates
$\R_{\text{conv},1}^{\text{free}}$ and $\boldS_{\text{conv},1}^{\text{free}}$.
Then we take the average of the first two dimensions to transform the discrepancy tensor
$(\W_{\text{conv},1}^{\text{free}} - \R_{\text{conv},1}^{\text{free}} \otimes \W_{\text{conv}}^{\text{base}} \otimes \boldS_{\text{conv},1}^{\text{free}})$
to be a $1\times 1\times J \times I$ tensor, which can then be 
applied a similar decomposition with the linear layers (Eq.~\ref{Eq.SVD_FC} and~\ref{Eq.lowrankWfree}.) 
to obtain the low-rank estimates of $\B_{\text{conv},1}^{\text{free}}$.

\subsection{Rank Selection for each layer based on Hessian Information}
\label{sec.SelectRank}
Sections~\ref{Sect:DecomposeFC} and \ref{Sect:DecomposeConv} present how the low-rank approximation can be used to transfer knowledge and 
reduce the number of parameters for a single layer across the tasks in continual learning. 
However, how to select the preserved rank for each layer remains unsolved. 
We tackle this problem by measuring how the empirical risk $\calL(\W)$ is 
influenced by the introduced low-rank parameters across different layers 
so that we can assign a larger rank to a layer that contributes
more to the risk. Inspired by the previous studies~\cite{dong2019hawq,dong2020hawq} about the relationship between Hessian and quantization errors,
we establish the following Theorem~\ref{thm: hessian-perturb-influence}. The full proof is presented in Appendix B.
\begin{theorem}
\label{thm: hessian-perturb-influence}
Assume that a neural network of $L$ layers with vectorized weights 
$(\bomega^{\star}_1, \dots, \bomega^{\star}_L)$ that have converged to local optima,
such that the first and second order
optimally conditions are satisfied, \emph{i.e.}, the gradient is zero, and the Hessian is positive semi-definite.
Suppose a perturbation $\Delta \bomega^{\star}_1$ applied to the first layer
weights, then we have the loss change
\begin{equation}
    \begin{split}
        |\mathcal{L}(\bomega^{\star}_1 -\Delta \bomega^{\star}_1,& \dots, \bomega^{\star}_L) - 
        \mathcal{L}(\bomega^{\star}_1, \dots, \bomega^{\star}_L)| \\ 
        &\leq \frac{1}{2}\|\bH_1\|_F \cdot \|\Delta \bomega^{\star}_1\|^2_F
        +o\left(\|\Delta \bomega^{\star}_1\|^2_F\right),
    \end{split}
\end{equation}
where $\bH_1=\nabla^2\mathcal{L}(\bomega^{\star}_1)$ is the Hessian matrix at only the variables of the first layer weights.
\end{theorem}
\textbf{Remark:} Theorem~\ref{thm: hessian-perturb-influence} demonstrated the relationship between the perturbation $\Delta w^{\star}_1$ on weights and the effect on the loss objective $\Delta \mathcal{L}$. Specifically, when a weight perturbation $\Delta w^{\star}_1$ is applied to the related weight matrix $w^{\star}_1$, the perturbation introduced on the loss function is upper bounded mainly by the product of the Frobenius norms of Hessian matrix (i.e., $\|\bH_1\|_F$) and weight perturbation (i.e., $\|\Delta \bomega^{\star}_1\|^2_F$). It further inspires us to follow this rule to select the proper ranks by considering the low-rank approximation in the previous section as a perturbation to the model weights.

In our low-rank perturbation setting,
we assume that $\W_1^{\text{free}}$ by warm-up training is a local optimum. By~Theorem~\ref{thm: hessian-perturb-influence},
we consider the difference between $\W_1^{\text{free}}$ and its $k$-rank approximation $\W_1^{(k){\text{free}}}$ as a perturbation $\Delta \W_1^{\text{free}}$ for the weights.
Then the amount of perturbation can be computed with the low-rank approximation error:

\begin{equation}
\begin{split}
    \|\Delta \W_1^{\text{free}}\|_F = & 
    \|\W_1^{\text{free}} - \W_1^{(k){\text{free}}}\|_F =\sqrt{\sum_{i=k+1}^{r}\sigma_i^2}\\
\end{split}
\label{Eq.kapprximatW}
\end{equation}
where $\{\sigma_i\}_{i=1}^{r}$ are the singular values of $\W_1^{\text{free}}$
and $r$ is the matrix rank of $\W_1^{\text{free}}$.

Thus, according to Theorem~\ref{thm: hessian-perturb-influence},
the influence on the loss introduced by this low-rank weight perturbation is given by 
\begin{equation}
    \begin{split}
        |\mathcal{L}(\W_1^{(k)\text{free}}) & - \mathcal{L}(\W_1^{\text{free}}) |\\ 
        &\leq \frac{1}{2}\|\bH_1\|_F \cdot \left(\sum_{i=k+1}^{r}\sigma_{i}^2\right)
        +o\left(\sum_{i=k+1}^{r}\sigma_{i}^2\right)
    \end{split}
    \label{Eq.influenceLoss}
\end{equation}
where the Hessian matrix $\bH_1$ can be approximated by the negative empirical Fisher information  \cite{kunstner2019limitations}, \emph{i.e.},
the outer product of the gradient vector for the layer weights.
So $\|\bH_1\|_F$ can be approximated by $\| \bg_1 \|_2^2$, 
where $\bg_1=\frac{\partial\mathcal{L}}{\partial \W_1}|_{\W_1=\W_1^{\text{free}}}$. 
Finally, we can quantitatively measure the contribution of the loss of
adding a marginal rank $k$ for a
particular layer $l$ by $\| \bg_l \|_2^2 \sigma_{l,k}^2$, 
where $\bg_l$ is the gradient for the layer-$l$ weights and
$\sigma_{l,k}$ is the $k$-th singular value of the free-trained layer-$l$ weights, and sort them by the descending order of importance. 

For a given loss approximation rate $\alpha$ (\emph{e.g.}, 0.9), we can determine the rank $k_l$ (with $k_l\leq r_l$ where $r_l$ is the total rank of the layer $l$) for each layer $l=1,...,L$ by solving
\begin{equation}\label{Eq.solveRank}
\begin{aligned}
        &\min_{k_1,...,k_{L}} \quad \sum_{l=1}^{L} \sum_{i=1}^{k_l} \|  \bg_l \|_2^2 \sigma_{l,i}^2\\
        & \text{s.t.} \sum_{l=1}^{L} \sum_{i=1}^{k_l} \|  \bg_l \|_2^2 \sigma_{l,i}^2 \geq \alpha \left( \sum_{l=1}^{L} \sum_{i=1}^{r_l} \|  \bg_l \|_2^2 \sigma_{l,i}^2 \right)
\end{aligned}
\end{equation}

\textbf{Remark:} Eq.~\ref{Eq.solveRank} enables a dynamic scheme for the trade-off between the approximation precision and computational efficiency. For a given approximation rate, the model can automatically select the rank for all the layers in the model. 

\subsection{Regularization and Pruning on Parameters}
\label{Sect:Pruning}
The proposed low-rank perturbation method introduced extra parameters compared to a single-task model. For these parameters, we can further add regularization to avoid overfitting. To control the model growth, we can further prune the introduced parameters to improve memory efficiency. 

Firstly, by following ~\cite{APD:YoonKYH20}, we can add regularizations on $\U^{(k) {\text{free}}}$, $\V^{(k){\text{free}}}$, $\R^{\text{free}}$,$ \boldS^{\text{free}}$ since second task $\calT_1$ to enhance the sparsity of the task-private parameters (see Appendix D for more discussion for this fine-tuning objective):
\begin{equation}\label{Eq:loss_reg}
    \begin{aligned}
         \calL_{\text{reg}}(\W_t; \lambda_0, \lambda_1) =\sum_{l}&\Big[\lambda_0 (  \| \U_{t,l}^{(k_l) \text{free}}  \| + \| \V_{t,l}^{(k_l) \text{free}}\|) \\ 
         & + \lambda_1 (\|\R_{t,l}^{\text{free}}\|_2^2 +\| \boldS_{t,l}^{\text{free}} \|_2^2 \\
         & + \| \U_{{t,l}}^{(k_l) \text{free}} \|_2^2 + \|\V_{t,l}^{(k_l) \text{free}} \|_2^2) \Big]
    \end{aligned}
\end{equation}
where $\lambda_0,\lambda_1$ are balancing coefficients, and the subscripts $(t,l)$ indicate the relevant weights for layer $l$ in task $t$.

Secondly, we can also prune the extra parameters by setting zero values for elements whose absolute values are lower than a certain threshold.
The threshold can be selected in the following three ways:
\begin{itemize}
    \item[(1)] Pruning via absolute value: 
    a fixed tiny positive value (\emph{e.g.}, $10^{-5}$) is set for the threshold.
    \item[(2)] Pruning via relative percentile: to control the ratio of increased parameter size over a single-task model size under $\gamma$,
    the pruning threshold is selected as the $(1-\gamma)$-percentile of the low-rank parameters among all layers of all tasks.
    \item[(3)] Pruning via mixing absolute value and relative percentile: we set a threshold as the maximum of the thresholds obtained from the above two methods to prune using relative percentiles.
\end{itemize}

\subsection{Summary of Algorithm}
\label{Sect.FullProcess}
 \begin{algorithm}[t]
    \caption{Hessian Aware Low-Rank Perturbation}
   \label{Alg.HALRP}
 	\begin{algorithmic}[1] 
 		\REQUIRE  Task data $\{\calD_t\}_{t=0}^{T-1}$; total epochs for one task $n$; rank estimation epochs $n_{r}$; parameter increments limitation ratio $p$, approximation rate $\alpha$.
         \ENSURE Base weights $\W^{\text{base}}$ and $\{\W_t^{\star}\}_{t=1}^{T-1}$ for each task.
         \STATE Obtain $\W^{\text{base}} = \argmin_{\W} \calL(\W; \mathcal{D}_0)$ on task $\mathcal{T}_0$.\\
         \FOR{$t=1,\cdots, T-1$}
        \STATE Warm-up pre-training on task $\mathcal{T}_t$ for $n_{r}$ epochs:  $\W^{\text{free}}_t=\arg \min_{\W} \calL(\W; \mathcal{D}_t)$.\\
        \STATE  Low-rank decomposition for all layers via Eq.~\ref{Eq.LowRankDecoposition} or Eq.~\ref{Eq.LowRankDecopositionCNN}: {\small $\W^{\text{free}}_{t} = \R^{\text{free}}_{t} \W^{\text{base}} \boldS^{\text{free}}_{t} + \B^{\text{free}}_{t}$}.
    \STATE Apply $\U_{t}^{\text{free}}, \bSigma_{t}^{\text{free}}, \V_{t}^{\text{free}} \leftarrow \texttt{SVD}(\B^{\text{free}}_{t})$.
        \STATE Select the ranks $k_l$ for each layer $l$ through Eq.~\ref{Eq.solveRank}.
    \STATE Re-initialize the task $\mathcal{T}_t$ parameters with Eq.~\ref{Eq.lowrankWfree}.
    \STATE Fine-tuning on $\mathcal{T}_t$ for $(n-n_r)$ epochs with:
        \begin{small}
        \begin{equation}
        \label{Eq.Crieterion}
        \W_t^{\star} = \arg\min_{\W} \left[\calL(\W; \mathcal{D}_t) + \calL_{\text{reg}}(\W; \lambda_0, \lambda_1) \right]
    \end{equation}
    \end{small}
        \STATE If the size of the introduced parameters is larger than a threshold $p$, apply the pruning method in Section~\ref{Sect:Pruning}.
        \ENDFOR
        \RETURN $\W^{\text{base}}$ and $\{\W_t^{\star}\}_{t=1}^{T-1}$.
        \end{algorithmic}
        \label{HAWLRP}
\end{algorithm}
In the previous sections, we described the model update from $\mathcal{T}_0$ to $\mathcal{T}_1$ with an illustrative example. The overall description of our proposed HALRP is shown in Algorithm~\ref{Alg.HALRP}. 

At the start, the learner was trivially trained on the first task $\mathcal{T}_0$ to obtain $\W^{\text{base}}$. As for each incoming task $\mathcal{T}_t$ with $t=1,...,T-1$, we first train the model without any constraints for $n_r$ epochs to get a rough initialization $\W^{\text{free}}_t$. Secondly, we apply the low-rank decomposition on all layers with Eq.~\ref{Eq.LowRankDecoposition} (for linear layers) or Eq.~\ref{Eq.LowRankDecopositionCNN} (for convolutional layers) by solving a least square error minimization problem described in Eq.~\ref{Eq.Solve_RSB}. Then, we further apply the singular value decomposition for the residual matrix $\B^{\text{free}}_t$. With this decomposition, we can measure the Hessian-aware perturbations and the rank $k_l$ for each layer $l$ through Eq.~\ref{Eq.solveRank}. After the rank selection, we can re-initialize the model parameters with the approximated weights $\W^{(k)\text{free}}_t \approx \W^{\text{free}}_t$ and then fine-tune the model for the remaining $(n-n_r)$ epochs to obtain the optimal weights for $\mathcal{T}_t$. As for the inference stage, the base weights and the task-specific parameters can be adopted to make predictions for each task.

\section{Experimental Results}
We first compare the accuracy over several recent baselines with standard CL protocol. Then, we studied the task order robustness, forgetting, memory cost, training time efficiency, and the ablation study to show the effectiveness further. 
We briefly describe the experimental setting herein while delegating more details in Appendix F.

\subsection{Experimental Settings}

\textbf{Datasets}:
We evaluate the algorithm on the following datasets: 1) \textbf{CIFAR100-Split}: Split the classes into 10 groups; for each group, consider a 10-way classification task. 2) \textbf{CIFAR100-SuperClass}: It consists of images from 20 superclasses of the CIFAR-100 dataset.
3) \textbf{Permuted MNIST (P-MNIST)}: obtained from the MNIST dataset by random permutations of the original MNIST pixels. We follow~\cite{ebrahimi2019uncertainty} to create 10 sequential tasks using different permutations, and each task has 10 classes. 4) \textbf{Five-dataset}: It uses a sequence of 5 different benchmarks including CIFAR10~\cite{krizhevsky2009learning}, MNIST~\cite{deng2012mnist}, notMNIST~\cite{bulatov2011notmnist}, FashionMNIST~\cite{xiao2017fashion} and SVHN~\cite{yuval2011reading}. Each benchmark contains 10 classes.
5)~\textbf{Omniglot Rotation} \cite{lake2015human}:  100 12-way classification tasks. The rotated images in $90^\circ$, $180^\circ$, and $270^\circ$ are generated by following~\cite{APD:YoonKYH20}. 
6)~\textbf{TinyImageNet}: a variant of ImageNet~\cite{deng2009imagenet} dataset containing 200 classes. Here, we adopted two settings, one with 20 10-way classification tasks (\textbf{TiyImageNet 20-split}) and another more challenging setting with 40 5-way classification tasks (\textbf{TinyImageNet 40-split}).

\textbf{Baselines}:
We compared the following baselines by the publicly released code or our re-implementation:
\textbf{STL}: single-task learning with individual models for each task.
\textbf{MTL}: multi-task learning with a single for all the tasks simultaneously~\cite{zhang2022SurveyMultiTask}.
\textbf{EWC}:~\emph{Elastic Weight Consolidation} method proposed by~\cite{kirkpatrick2017overcoming}.~\textbf{L2}: the model is trained with $L_2$-regularizer~\cite{kirkpatrick2017overcoming} $\lambda \| \btheta_t -\btheta_{t-1} \|_{2}^2$ between the current model and the previous one.
\textbf{BN}: The \emph{Batch Normalization} method~\cite{ioffe2015batch}.
\textbf{BE}: The \emph{Batch Ensemble} method proposed by~\cite{wen2020batchensemble}.
\textbf{APD}: The \emph{Additive Parameter Decomposition} method~\cite{APD:YoonKYH20}. Each layer of the target network was decomposed into task-shared and task-specific parameters with mask vectors.
\textbf{APDfix}: We modify the APD method by fixing the model parameters while only learning the mask vector when a new task comes to the learner.
\textbf{IBPWF}: Determine the model expansion with non-parametric Bayes and weights factorization~\cite{mehta2021continual}.
\textbf{GPM}~\cite{saha2021gradient}: A replay-based method by orthogonal gradient descent.
\textbf{WSN}~\cite{kang2022forget}: Introduce learnable weight scores to generate task-specific binary masks for optimal subnetwork selection. 
\textbf{BMKP}~\cite{sun2023decoupling}: A bilevel memory framework for knowledge projection: a working memory to ensure plasticity and a long-term memory to guarantee stability.
\textbf{CLR}~\cite{ge2023clr}: An \emph{expansion-based} method by applying Channel-Wise Linear Reprogramming transformations on each convolutional layer in the base model. This method originally relies on a model pre-trained on an extra dataset (i.e., ImageNet-1K) disjoint with the above task datasets, rather than the \emph{train-from-scratch} manner adopted by other baselines. To make a fair comparison, we pre-trained the base model on the first task of the above task datasets and applied it to all tasks.
\textbf{PRD}~\cite{asadi2023prototype}: Prototype-sample relation distillation with supervised contrastive learning.
Furthermore, we also compared our methods with \textbf{IRU}~\cite{hyder2022incremental}, a method also based on the low-rank decomposition. However, due to the code limitation\footnote{IRU code only contains the implementation on the multi-layer perceptron. See https://github.com/CSIPlab/task-increment-rank-update} of \textbf{IRU}~\cite{hyder2022incremental}, we further implemented our method under the dataset protocol and model architecture setting in~\cite{hyder2022incremental} and compared our performance with the results reported in~\cite{hyder2022incremental}. For the implementation of all the methods, we applied the same hyperparameters (e.g., batch size, training epochs, regularization coefficient) to realize fair comparisons. See Appendix F.3 for more details.

\textbf{Model Architecture}: For CIFAR100-Split, CIFAR100-SuperClass, and P-MNIST datasets, we adopted LeNet as the base model. As for Five-dataset and ImageNet datasets, we evaluated with AlexNet and reduced ResNet18 networks wehre the latter has reduced filters compared to standard ResNet18 (see Appendix F.2). And we followed~\cite{kang2022forget} to use an extended LeNet model on the Omniglot-Rotation dataset.

\textbf{Evaluation Metrics}: We mainly adopted the average of the accuracies of the final model on all tasks (we simply call ``accuracy'' or ``Acc.'' later) for the empirical comparisons. As for the forgetting statistics, we applied backward transfer (BWT)~\cite{saha2021gradient} as a quantitive measure. Especially, we also compared the order-robustness of different methods by calculating the \emph{Order-normalized Performance Disparity} that will be introduced later. For each single experiment (\emph{e.g.}, each task order), we repeat with five random seeds to compute the average and standard error.

\begin{table*}[t]
\centering
\resizebox{0.75\linewidth}{!}{\Large
\begin{tabular}{@{}l|cccc|cccc@{}}
\toprule
\multicolumn{1}{c|}{} & \multicolumn{4}{c|}{CIFAR100-Split (with LeNet)}                                      & \multicolumn{4}{c}{CIFAR100-SuperClass (with LeNet)}                                  \\ 
Method                & 5\%              & 25\%             & 50\%             & 100\%             & 5\%              & 25\%             & 50\%             & 100\%              \\ \midrule
STL                   & 45.13 $\pm$ 0.04 & 59.04 $\pm$ 0.03 & 64.38 $\pm$ 0.06 & 69.55 $\pm$ 0.06 & 43.76 $\pm$ 0.68 & 56.09 $\pm$ 0.07 & 60.06 $\pm$ 0.06 & 64.47 $\pm$ 0.05 \\
MTL                   & 44.95 $\pm$ 0.11 & 60.21 $\pm$ 0.28 & 65.65 $\pm$ 0.20 & 69.70 $\pm$ 0.28 & 40.43 $\pm$ 0.15 & 49.88 $\pm$ 0.27 & 53.83 $\pm$ 0.27 & 55.62 $\pm$ 0.41 \\ \midrule
L2                    & 37.15 $\pm$ 0.21 & 48.86 $\pm$ 0.28 & 53.35 $\pm$ 0.34 & 58.09 $\pm$ 0.43 & 34.03 $\pm$ 0.08 & 43.40 $\pm$ 0.27 & 46.10 $\pm$ 0.28 & 48.75 $\pm$ 0.24 \\
EWC                   & 37.76 $\pm$ 0.20 & 50.09 $\pm$ 0.38 & 55.65 $\pm$ 0.40 & 60.53 $\pm$ 0.26 & 33.70 $\pm$ 0.32 & 44.02 $\pm$ 0.39 & 47.35 $\pm$ 0.47 & 49.97 $\pm$ 0.39 \\
BN                    & 37.60 $\pm$ 0.17 & 50.70 $\pm$ 0.28 & 54.79 $\pm$ 0.28 & 60.34 $\pm$ 0.40 & 36.76 $\pm$ 0.14 & 48.20 $\pm$ 0.16 & 51.43 $\pm$ 0.16 & 55.44 $\pm$ 0.19 \\
BE                    & 37.63 $\pm$ 0.15 & 51.13 $\pm$ 0.3  & 55.37 $\pm$ 0.28 & 61.09 $\pm$ 0.33 & 37.05 $\pm$ 0.20 & 48.48 $\pm$ 0.14 & 51.78 $\pm$ 0.17 & 55.97 $\pm$ 0.17 \\
APD                   & 36.60 $\pm$ 0.14 & 54.59 $\pm$ 0.07 & 59.71 $\pm$ 0.03 & 66.54 $\pm$ 0.03 & 32.81 $\pm$ 0.29 & 49.00 $\pm$ 0.06 & 52.64 $\pm$ 0.19 & 60.54 $\pm$ 0.23 \\
APDfix                & 35.66 $\pm$ 0.33 & 54.62 $\pm$ 0.11 & 59.86 $\pm$ 0.24 & 66.64 $\pm$ 0.14 & 24.27 $\pm$ 0.22 & 48.71 $\pm$ 0.11 & 53.42 $\pm$ 0.12 & 61.47 $\pm$ 0.16 \\
IBWPF                 & 38.35 $\pm$ 0.26     & 47.87 $\pm$ 0.25       & 53.46 $\pm$ 0.13      & 57.13 $\pm$0.15    & 33.09 $\pm$ 0.50      & 51.32 $\pm$ 0.27     & 52.52 $\pm$ 0.26     & 55.98 $\pm$ 0.33     \\
GPM                 &  32.86 $\pm$ 0.35      & 51.61 $\pm$ 0.22       & 57.60 $\pm$ 0.19      & 64.49 $\pm$ 0.10       &  34.88 $\pm$ 0.30      & 47.31 $\pm$ 0.51     & 51.23 $\pm$ 0.55      &  57.91 $\pm$ 0.28      \\
WSN                   & 37.01 $\pm$ 0.63       & 55.21 $\pm$ 0.59       & 61.56 $\pm$ 0.42       & 66.56 $\pm$ 0.49       & 36.89 $\pm$ 0.49       & 52.42 $\pm$ 0.62      & 58.23 $\pm$ 0.38      & 61.81 $\pm$ 0.54      \\ 
BMKP                   & 42.36 $\pm$ 0.90       & 56.81 $\pm$ 1.05       & 62.87 $\pm$ 0.60      & 66.95 $\pm$ 0.53       & 37.26 $\pm$ 0.87       & 53.62 $\pm$ 0.59      & 57.76 $\pm$ 0.66      & 61.97 $\pm$ 0.19      \\ 
CLR                   & 36.46 $\pm$ 0.29       & 51.44 $\pm$ 0.35       & 57.00 $\pm$ 0.43      & 61.83 $\pm$ 0.60       & 37.93 $\pm$ 0.25       & 49.82 $\pm$ 0.56      & 53.86 $\pm$ 0.63      & 57.07 $\pm$ 0.60      \\ 
PRD                   & 31.58 $\pm$ 0.29       & 56.07 $\pm$ 0.22       & 59.61 $\pm$ 0.33      & 62.74 $\pm$ 0.45       & 33.34 $\pm$ 0.48       & 52.99 $\pm$ 0.28      & 55.85 $\pm$ 0.45      & 57.80 $\pm$ 0.50      \\ 
\midrule
HALRP                & \textbf{45.09 $\pm$ 0.05}     & \textbf{58.94 $\pm$ 0.09} & \textbf{63.61 $\pm$ 0.08} & \textbf{67.92 $\pm$ 0.17} & \textbf{43.84 $\pm$ 0.04} & \textbf{54.93 $\pm$ 0.04} & \textbf{58.68 $\pm$ 0.11} & \textbf{62.56 $\pm$ 0.30} \\ \bottomrule
\end{tabular}
}
\caption{Accuracies$\uparrow$ on CIFAR100-Split/-SuperClass with different percentages of training data.}
\label{tab:experiments_cifar}
\end{table*}

\subsection{Empirical Accuracy}
\label{Sect:EmpiricalAcc}
We provide the average accuracies on the six benchmarks in Table~\ref{tab:experiments_cifar} and Table~\ref{tab:PMNIST_FIVE_OMNI}. 
From these numerical results, we can conclude our method achieved state-of-the-art performance compared to the baseline methods.

According to the evaluations on CIFAR100-Split and CIFAR100-SuperClass in Table~\ref{tab:experiments_cifar}, we can observe that our proposed HALRP outperforms the recent methods (\emph{e.g.}, GPM (replay-based), APD and WSN (expansion-based)) with a significant margin (\emph{e.g.}, with an improvement about $1\%\sim3\%$). Especially, we also evaluated the performances of all the methods under different amounts of training data (\emph{i.e.}, $5\%\sim100\%$) on these two benchmarks. It was interesting to see that our proposed method had advantages in dealing with extreme cases with limited data. Compared to other methods, the performance margins become more significant with fewer training data. For example, on CIFAR100-Split, our proposed HALRP outperforms APD and WSN with an improvement of $1.28\%$ and $1.36\%$ respectively, but these margins will dramatically rise to $\sim8\%$ if we reduce the training set to $5\%$ of the total data. These results showed that our method can work better in these extreme cases.

On Five-dataset, we adopted two types of neural network, \emph{i.e.}, AlexNet and ResNet18, to verify the effectiveness of HALRP under different backbones. We can observe that our method consistently outperforms the other methods under different backbones, indicating the applicability and flexibility for different model choices. Compared with the analyses in the following part, we can conclude that our method is also robust under different orders of tasks.

Especially, we also conducted experiments on the Omniglot-Ratation dataset to demonstrate the scalability of our method. We follow~\cite{kang2022forget} to learn the 100 tasks under the default sequential order. The average accuracy was shown in Table~\ref{Tab.SubOmniglot}. We demonstrate that our method is applicable to a large number of tasks and can still achieve comparable performance with limited model increment.

On TinyImageNet dataset, we evaluated our methods on two settings, one with 20-split and another with 40-split, and the results are posted in Table~\ref{Tab.SubTiny}. According to the empirical results, we can conclude that our proposed method can perform well on this challenging dataset, with a better trade-off between the average accuracy and task order robustness. Compared to the previous methods APD and APDfix that aim to address the issue of task order robustness, our HALRP can perform well in a more consistent manner. Furthermore, our methods also have advantages regarding computational efficiency and memory consumption that we will discuss later in the following sections.

Additionally, we also provide comparisons with IRU~\cite{hyder2022incremental} that is also based on low-rank decomposition. We note that the official code of IRU only contains a demo for multi-layer perceptron (MLP), and no complete code for the convolutional network (even LeNet) can be found. Due to this limitation, we cannot implement IRU under our setting. To make fair comparisons, we re-implement our method under the setting of IRU~\cite{hyder2022incremental}. We describe the setting as follows: 
1) Dataset Protocol: \cite{hyder2022incremental} generated 20 random tasks on PMNIST, rather than 10 tasks in this work. Moreover, \cite{hyder2022incremental} randomly divided the CIFAR100 dataset into 20 tasks, rather than the superclass-based splitting in our paper. Thus, we denote the two datasets in~\cite{hyder2022incremental} as ``PMNIST-20" and ``CIFAR100-20'', to distinguish them from the ``PMNIST" and ``CIFAR100-SuperClass'' used in the common setting.
(2) Model Architecture: Apart from the architectures we introduced before, we also follow the MLP used in~\cite{hyder2022incremental}, a three-layer (fully-connected) multilayer perceptron with 256 hidden nodes. The results of our method and IRU are listed in Table~\ref{tab:res-iru}. We can conclude that our method achieved a large performance gain compared to IRU.

In the following parts, we will further investigate the empirical performance under different aspects, i.e., task order robustness, forgetting statistics, model growth, and time complexity. We will show that our proposed method can achieve a better trade-off among these realistic metrics apart from the average accuracy.

\begin{table*}[t]
\centering
\begin{subtable}[b]{0.25\linewidth}
\centering
\resizebox{0.9\linewidth}{!}{\Large
\begin{tabular}{@{}l|ccc@{}}
\toprule
    & \multicolumn{3}{c}{P-MNIST} \\
    & \multicolumn{3}{c}{LeNet} \\ \cmidrule(l){2-4} 
\multicolumn{1}{c|}{Method} & Acc.$\uparrow$& MOPD$\downarrow$&AOPD$\downarrow$\\ \midrule
STL& 98.24$\pm$0.01& 0.15& 0.09\\
MTL& 96.70$\pm$0.07&1.58&0.81\\ \hline
L2&79.14$\pm$0.70&29.66&18.94\\
EWC&81.69$\pm$0.86&21.51&12.16\\
BN&81.04$\pm$0.15& 19.77&8.18\\
BE&83.80$\pm$0.08&16.76&6.88\\
APD&97.94$\pm$0.02&0.25&0.16\\
APDfix&97.99$\pm$0.01& 0.10&0.11\\
GPM& 96.69$\pm$0.02 & 0.45 & 0.27\\
WSN& 97.91$\pm$0.02 & 0.36  & 0.22\\ 
BMKP& 97.08$\pm$0.01 & 3.21  & 1.04\\
CLR& 88.55$\pm$0.20 & 14.97  & 7.88\\
PRD& 83.16$\pm$0.17 & 7.91  & 6.16\\
\midrule                    
HALRP&\multicolumn{1}{c}{\textbf{98.10$\pm$0.03}}& \multicolumn{1}{c}{0.47} & \multicolumn{1}{c}{0.24} \\ \bottomrule
\end{tabular}
}
\caption{Results on P-MNIST.}
\label{Tab.SubPMNIST}
\end{subtable}
\centering
\begin{subtable}[b]{0.45\linewidth}
\centering
\resizebox{0.9\linewidth}{!}{\Large
\begin{tabular}{@{}l|cccccc@{}}
\toprule
    & \multicolumn{6}{c}{Five-dataset}   \\ \cmidrule(l){2-7} 
    & \multicolumn{3}{c|}{AlexNet}   & \multicolumn{3}{c}{ResNet-18}    \\ \cmidrule(l){2-7} 
\multicolumn{1}{c|}{Method} & Acc.$\uparrow$& MOPD$\downarrow$& \multicolumn{1}{c|}{AOPD$\downarrow$} & \multicolumn{1}{c}{Acc.$\uparrow$} & \multicolumn{1}{c}{MOPD$\downarrow$} & \multicolumn{1}{c}{AOPD$\downarrow$} \\ \midrule
STL           &89.32$\pm$0.06& 0.74& \multicolumn{1}{c|}{0.30}& 94.24$\pm$0.05&0.67&0.24\\
MTL             &88.02$\pm$0.18&2.08& \multicolumn{1}{c|}{0.68}&93.82$\pm$0.06&0.67&0.30\\ \hline
L2          & 78.24$\pm$2.00 & 35.11&\multicolumn{1}{c|}{15.35}&85.94$\pm$2.79&37.03&12.72\\
EWC         & 78.44$\pm$2.20&33.29& \multicolumn{1}{c|}{13.18}&86.32$\pm$2.80&33.84&11.80\\
BN      &82.35$\pm$2.45&34.83& \multicolumn{1}{c|}{12.56}&88.36$\pm$2.21&30.22&9.55\\
BE  &82.91$\pm$2.37&33.91& \multicolumn{1}{c|}{11.63}&88.75$\pm$2.14&29.22& 8.97\\
APD     &83.70$\pm$0.90&4.80& \multicolumn{1}{c|}{3.45}&92.18$\pm$0.28&3.50  &1.54\\
APDfix      &84.03$\pm$1.24&5.50& \multicolumn{1}{c|}{3.66}&91.91$\pm$0.48      & 6.74&1.98\\
GPM   & 87.27$\pm$0.61  & 4.54 & \multicolumn{1}{c|}{1.88} &  88.52$\pm$0.28  & 6.97 & 2.82\\
WSN          &86.74$\pm$0.40 & 8.54 &\multicolumn{1}{c|}{2.89}&92.58$\pm$0.39 & 4.62 & 1.21\\ 
BMKP         &84.03$\pm$0.55 & 9.32 &\multicolumn{1}{c|}{3.07}& 92.57$\pm$0.65 & 9.08  & 2.13 \\ 
CLR         &86.68$\pm$1.41 & 19.78 &\multicolumn{1}{c|}{7.08}& 90.04$\pm$1.04 & 14.05  & 4.51 \\ 
PRD         & 74.74$\pm$0.69 & 17.53 &\multicolumn{1}{c|}{9.23}& 88.45$\pm$0.93 & 14.17  & 5.37 \\ 
\midrule
HALRP                      &\multicolumn{1}{c}{\textbf{88.81$\pm$0.31}} & \multicolumn{1}{c}{4.28}&\multicolumn{1}{c|}{1.31}&\textbf{93.39$\pm$0.30} & 4.39 & 1.27\\ \bottomrule
\end{tabular}
}
\caption{Results on Five-dataset with different backbones.}
\label{Tab.SubFive}
\end{subtable}
\centering
\begin{subtable}[b]{0.25\linewidth}
\centering
\resizebox{0.9\linewidth}{!}{\Large

\begin{tabular}{@{}l|ccc@{}}
\toprule
    & \multicolumn{3}{c}{Omniglot-Rotation}   \\
    & \multicolumn{3}{c}{LeNet}    \\ 
    \cmidrule(l){2-4} 
\multicolumn{1}{c|}{Method} & Acc.$\uparrow$& MOPD$\downarrow$&AOPD$\downarrow$\\ 
\midrule
STL& 80.93$\pm$0.18 & 20.83 & 3.42\\
MTL& 93.95$\pm$0.11& 6.25 & 2.11 \\ 
\hline
L2& 69.86$\pm$1.23 & 17.23 & 6.96 \\
EWC& 69.75$\pm$1.28 & 21.39 & 7.02 \\
BN& 77.08$\pm$0.86 & 14.41 & 5.58 \\
BE& 78.24$\pm$0.69 & 17.17 & 5.53\\
APD($\star$)& 81.60$\pm$0.53 & 8.19 & 3.78\\
APDfix& 78.14$\pm$0.12 & 6.53 & 2.63\\
GPM& 80.41$\pm$0.16 & 28.33 & 13.02 \\
WSN& 82.55$\pm$0.44 & 17.09 & 7.57 \\
BMKP& 81.12$\pm$2.71 & 26.53 & 16.17\\
CLR& 72.75$\pm$1.41 & 24.30 & 12.27\\
PRD& 74.49$\pm$2.78 & 49.17 & 18.44\\
\midrule                    
HALRP& \textbf{83.08$\pm$0.73} & 10.36 & 3.91 \\ 
\bottomrule
\end{tabular}
}
\caption{Results on Omniglot.}
\label{Tab.SubOmniglot}
\end{subtable}

\centering
\begin{subtable}[b]{0.9\linewidth}
\centering
\resizebox{0.9\linewidth}{!}{\Large
\begin{tabular}{@{}l|ccc|ccc|ccc|ccc@{}}
\toprule
    & \multicolumn{6}{c|}{TinyImageNet 20-split} & \multicolumn{6}{c}{TinyImageNet 40-split} \\ 
    \cmidrule(l){2-13} 
    & \multicolumn{3}{c|}{AlexNet} & \multicolumn{3}{c|}{ResNet18} & \multicolumn{3}{c|}{AlexNet} & \multicolumn{3}{c}{ResNet18} \\ 
    \cmidrule(l){2-13} 
Method & Acc.$\uparrow$ &  MODP$\downarrow$   & AODP$\downarrow$    & Acc.$\uparrow$ &  MODP$\downarrow$   & AODP$\downarrow$ & Acc.$\uparrow$ &  MODP$\downarrow$   & AODP$\downarrow$    & Acc.$\uparrow$ &  MODP$\downarrow$   & AODP$\downarrow$    \\ \midrule
STL& 66.78$\pm$0.18 & 6.20  &  \multicolumn{1}{c|}{3.43} & 67.00$\pm$0.30 & 5.87  & 2.87  & 74.65$\pm$0.17 & 8.16  &  \multicolumn{1}{c|}{4.36} & 74.08$\pm$0.25 & 8.13  &  4.08 \\
MTL& 71.23$\pm$0.54 & 6.87  &  \multicolumn{1}{c|}{3.42} & 73.64$\pm$0.46 & 6.94  & 3.31  & 78.80$\pm$0.25 & 7.74  &  \multicolumn{1}{c|}{3.92} & 80.05$\pm$0.22 & 9.07  &  4.04 \\ \hline
L2& 56.33$\pm$0.22 & 11.93  & \multicolumn{1}{c|}{5.72} &  60.80$\pm$0.56 & 8.27  & 4.23 & 63.47$\pm$1.35 & 16.27  & \multicolumn{1}{c|}{7.42} &  65.59$\pm$1.03 & 14.80  & 7.00 \\
EWC& 56.55$\pm$0.22  &  10.93 & \multicolumn{1}{c|}{5.53} & 60.88$\pm$0.56  & 7.54  & 4.63 & 64.11$\pm$1.36  & 17.87  & \multicolumn{1}{c|}{6.91} & 66.54$\pm$0.94  & 14.67  & 6.99\\
BN& 57.20$\pm$0.11 &  12.07  & \multicolumn{1}{c|}{5.37} & 61.03$\pm$0.65  & 9.73 & 4.24 & 64.04$\pm$1.15 & 19.07   & \multicolumn{1}{c|}{7.23} & 66.17$\pm$1.75  & 15.33 & 6.92\\
BE& 57.62$\pm$0.41   & 11.80  & \multicolumn{1}{c|}{4.98} & 61.52$\pm$0.68  & 7.00 & 4.09 & 64.70$\pm$1.08   &  23.47 & \multicolumn{1}{c|}{6.97} & 66.77$\pm$1.61  & 16.94 & 7.07\\
APD& \textbf{67.26$\pm$0.38} &  12.00  & \multicolumn{1}{c|}{5.43} & 68.76$\pm$0.58  & 9.86  & 4.33 & 73.88$\pm$0.46 & 8.53   & \multicolumn{1}{c|}{5.04} & 73.85$\pm$0.40  & 11.46  & 5.57 \\
APDfix& 63.59$\pm$0.25   & 5.40   & \multicolumn{1}{c|}{3.68} & 67.69$\pm$0.41  & 6.54  & 3.64 & 67.91$\pm$1.33   & 23.47   & \multicolumn{1}{c|}{8.95} & 58.11$\pm$2.11  & 26.93  & 11.70\\
GPM& 60.84$\pm$0.32  & 11.00  & \multicolumn{1}{c|}{4.44} & 48.09$\pm$1.07  & 9.27 & 4.80 & 70.14$\pm$0.38  & 10.90  & \multicolumn{1}{c|}{4.52} & 43.40$\pm$7.68  & 48.80  & 38.80\\
WSN& 65.73$\pm$0.21  & 5.06  & \multicolumn{1}{c|}{3.31} & 68.27$\pm$0.47  & 8.40  & 4.52 & 73.46$\pm$1.27  & 8.00  & \multicolumn{1}{c|}{4.23} & 75.29$\pm$0.46  & 12.27  & 4.75\\
BMKP& 65.01$\pm$0.41  &  5.87 & \multicolumn{1}{c|}{3.37} & 67.45$\pm$0.59  & 10.60  & 6.69 & 73.57$\pm$0.21  & 9.60  & \multicolumn{1}{c|}{4.31} & 74.84$\pm$0.63  & 12.90  & 4.95\\
CLR& 57.29$\pm$0.39  & 8.80  & \multicolumn{1}{c|}{4.42} & 61.77$\pm$0.47  & 10.27  & 5.42 & 65.22$\pm$0.49  & 9.86  & \multicolumn{1}{c|}{5.74} & 67.59$\pm$0.92  & 16.54  & 9.30\\
PRD& 46.49$\pm$0.30  & 15.40  & \multicolumn{1}{c|}{8.33} & 49.65$\pm$0.57  & 19.33  & 11.59 & 53.54$\pm$0.45  & 19.74  & \multicolumn{1}{c|}{10.02} & 63.78$\pm$0.59  & 20.27  & 7.93\\
\midrule
HALRP& 66.68$\pm$0.20 & 5.60  & \multicolumn{1}{c|}{3.43} & \textbf{70.09$\pm$0.29}  & 4.67  & 3.20 & \textbf{74.01$\pm$0.25} &  7.47 & \multicolumn{1}{c|}{4.12} & \textbf{75.53$\pm$0.50}  & 7.87  & 4.48
\\ \bottomrule
\end{tabular}
}
\caption{Results on TinyImageNet with different backbones.}
\label{Tab.SubTiny}
\end{subtable}

\caption{Performance on P-MNIST, Five-dataset, Omniglot-Rotation and TinyImageNet. We ran experiments with five different task orders generated by different seeds on first two datasets. As for Omniglot-Rotation, we follow~\cite{kang2022forget} to show the scalability under the original sequential order. Acc.$\uparrow$ refers to the empirical accuracy, MOPD$\downarrow$ and AOPD$\downarrow$ refer to the task order robustness as discussed in Section~\ref{Sect:OPD}. Results with $\star$ are from \cite{APD:YoonKYH20}.}
     \label{tab:PMNIST_FIVE_OMNI}
\end{table*}

\begin{table}[t]{}
    \centering
\resizebox{0.8\linewidth}{!}{\Large
\begin{tabular}{c|c|c|c}
\hline
        & \makecell{PMNIST-20\\(MLP)} & \makecell{CIFAR100-20\\(ResNet18)} & \makecell{CIFAR100-20\\(MLP)} \\ \hline
Multitask($\star$) & $96.8$  & $70.2$  & $16.4$ \\ \hline
IRU($\star$)      & $85.60\pm0.15$ & $68.46\pm2.52$ & $65.90\pm2.16$  \\ \hline
Ours            & $92.02\pm0.04$ & $73.71\pm0.87$ & $67.21\pm0.68$  \\ \hline
\end{tabular}
}
\caption{Accuracies under the IRU~\cite{hyder2022incremental} setting. Results with $\star$ are reported in~\cite{hyder2022incremental}.}
\label{tab:res-iru}
\end{table}

\subsection{Robustness on Task Orders}
\label{Sect:OPD}
We evaluate the robustness of these algorithms under different task orders. Following the protocol of~\cite{APD:YoonKYH20}, we assessed the task order robustness with the \textbf{Order-normalized Performance Disparity} (OPD) metric, which is computed as the disparity between the performance $\Bar{P}_t$ of task $t$ under $R$ different task orders: $\mathrm{OPD}_t \triangleq \max \{\Bar{P}_t^1,\dots, \Bar{P}_t^R\} - \min \{\Bar{P}_t^1,\dots, \Bar{P}_t^R\}$.
The maximum OPD (MOPD) and average OPD (AOPD) are defined by 
\begin{equation}
    \begin{aligned}
        \mathrm{MOPD} &\triangleq \max \{\mathrm{OPD}_0,\dots, \mathrm{OPD}_{T-1}\}\\
        \mathrm{AOPD} &\triangleq \frac{1}{T}\sum_{t=0}^{T-1} \mathrm{OPD}_t
    \end{aligned}
\end{equation}
respectively.

Thus, smaller AOPD and MOPD indicate better task-order robustness. For CIFAR100-Split and SuperClass datasets, we adopted different orders by  following~\cite{APD:YoonKYH20}. For the P-MNIST and Five-dataset, we report the averaged results over five different task orders. The results on MOPD and AOPD on P-MNIST, Five-dataset, Omniglot-Rotation, TinyImageNet, and CIFAR100-Splits/-SuperClass are reported in Table~\ref{Tab.SubPMNIST}, Table~\ref{Tab.SubFive}, Table~\ref{Tab.SubOmniglot}, Table~\ref{Tab.SubTiny} and Table~\ref{Tab.CIFAR-OPD}, respectively. According to these results, we can generally conclude that our method can achieve a better accuracy-robustness trade-off, compared to the recent baseline APD.

\begin{table*}[]
\centering
\resizebox{0.9\linewidth}{!}{\Huge
\begin{tabular}{@{}l|cccccccc|cccccccc@{}}
\toprule
\multicolumn{1}{c|}{} & \multicolumn{8}{c|}{CIFAR100-Split (with LeNet)}                                                                       & \multicolumn{8}{c}{CIFAR100-SuperClass (with LeNet)}                                                                  \\
                      & \multicolumn{2}{c}{5\%} & \multicolumn{2}{c}{25\%} & \multicolumn{2}{c}{50\%} & \multicolumn{2}{c|}{100\%} & \multicolumn{2}{c}{5\%} & \multicolumn{2}{c}{25\%} & \multicolumn{2}{c}{50\%} & \multicolumn{2}{c}{100\%} \\ \midrule
Method                & MOPD$\downarrow$        & AOPD$\downarrow$      & MOPD$\downarrow$        & AOPD$\downarrow$       & MOPD$\downarrow$        & AOPD$\downarrow$       & MOPD$\downarrow$         & AOPD$\downarrow$        & MOPD$\downarrow$       & AOPD$\downarrow$       & MOPD$\downarrow$        & AOPD$\downarrow$       & MOPD$\downarrow$        & AOPD$\downarrow$       & MOPD$\downarrow$         & AOPD$\downarrow$       \\ \midrule
STL                   & 1.96        & 1.38      & 3.44        & 2.41       & 3.68        & 2.57       & 4.38         & 3.13        & 2.52       & 1.48       & 3.64        & 1.44       & 2.52        & 1.48       & 2.64         & 1.45       \\
MTL                   & 1.44        & 0.93      & 1.34        & 0.87       & 1.66        & 0.71       & 1.54         & 0.84        & 6.96       & 2.66       & 12.04       & 4.89       & 13.96       & 5.74       & 13.48        & 6.29       \\ \hline
L2                    & 11.66       & 6.13      & 13.14       & 6.96       & 15.02       & 7.50       & 15.18        & 7.15        & 9.92       & 4.55       & 13.24       & 6.43       & 13.32       & 7.04       & 14.44        & 6.82       \\
EWC                   & 11.92       & 6.07      & 13.20       & 6.85       & 13.78       & 6.96       & 13.44        & 6.48        & 13.24      & 5.60       & 10.92       & 6.53       & 13.48       & 8.03       & 21.60        & 8.32       \\
BN                    & 11.70       & 5.95      & 13.28       & 7.37       & 12.40       & 6.91       & 13.92        & 7.90        & 8.52       & 3.35       & 11.20       & 3.99       & 11.64       & 4.05       & 11.40        & 4.25       \\
BE                    & 12.12       & 5.47      & 12.92       & 6.13       & 9.28        & 5.56       & 8.50         & 5.35        & 9.56       & 3.46       & 11.88       & 4.00       & 11.80       & 3.88       & 10.84        & 4.03       \\
APD                   & 11.42       & 7.21      & 6.48        & 3.77       & 8.00        & 4.10       & 8.88         & 4.07        & 26.56      & 12.98      & 8.64        & 4.39       & 8.28        & 4.48       & 6.72         & 3.26       \\
APDfix                & 6.64        & 4.39      & 8.32        & 4.94       & 8.92        & 5.54       & 7.40         & 4.21        & 9.04       & 4.90       & 10.28       & 5.68       & 8.56        & 5.32       & 6.04         & 2.70       \\
IBPWF                 &   4.30          &  2.84         &  4.60           &  2.73          & 5.40            & 3.07           &  3.68             &  2.68           &  14.84          & 7.45           &   4.44          &  2.45          &  5.36           &  2.86          &   5.52           &  3.38          \\
GPM  &    11.14     &    6.37       &   6.32       &   4.22     &    6.32     &  3.69          &   2.28            &      1.348     & 9.32         &  4.46       &   10.00       & 5.53     & 8.84          &  5.89        &   7.68   & 4.47     \\
WSN                   & 4.16        & 2.57      & 4.42        & 2.71       & 4.2         & 2.62       & 3.56         & 2.39        & 5.08       & 3.14       & 4.76        & 3.192      & 4.00        & 2.16       & 3.76         & 2.36       \\ 
BMKP                   & 12.98        & 7.63     &13.22         &6.50        & 6.52          &4.30        &8.34          & 3.35        & 11.72       & 6.03       & 13.32        & 5.27      & 11.08        & 4.43        & 3.72         & 1.99       \\
CLR                   & 13.50        & 7.14     &12.77         &6.13        & 8.90          &5.82        &9.37          & 5.34        & 10.00       & 4.01       & 7.67        & 3.84      & 6.60        & 3.71        & 9.00         & 4.06       \\
PRD       & 5.80        & 3.80     & 4.60      & 2.66       & 4.16       & 2.63       &  5.20       & 3.16      &  7.47     &  4.28    &  8.33     & 3.95     &  5.33      & 3.29       &   5.00      &  2.86     \\
\midrule
HALRP                & 2.56        & 1.34      & 2.58        & 1.44       & 2.34        & 1.71       & 3.90         & 2.56        & 2.96       & 1.65       & 3.40        & 1.91       & 4.48        & 1.65       & 4.34         & 1.96       \\ \bottomrule
\end{tabular}
}
\caption{Task order robustness evaluation on CIFAR100-Split/SuperClass with different amounts of training data.}
\label{Tab.CIFAR-OPD}
\end{table*}

\subsection{Handling the Catastrophic Forgetting}
We then evaluated the abilities of different methods for overcoming catastrophic forgetting. We illustrate the average forgetting (BWT) on the CIFAR100 Split/SuperClass dataset under different amounts of data in Fig.~\ref{fig:comparison_forgetting}. Our proposed method can effectively address the forgetting issue, comparable with some state-of-the-art methods like WSN. We further provide detailed forgetting statistics within five different task orders in Appendix E.2.

\begin{figure}
    \centering
     \begin{subfigure}[b]{0.9\linewidth}
         \centering
         \includegraphics[width=\linewidth]{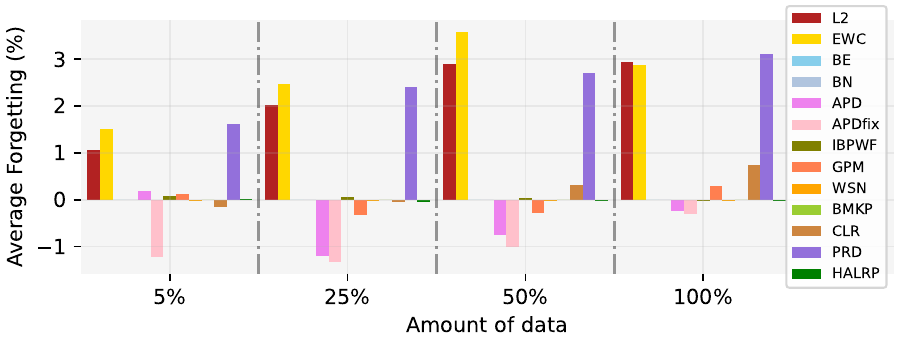}
         \caption{Forgetting on 10 tasks in CIFAR100-Splits}
         \label{fig:forgetting_c100splits}
     \end{subfigure}
    \centering
     \begin{subfigure}[b]{0.9\linewidth}
         \centering
         \includegraphics[width=\linewidth]{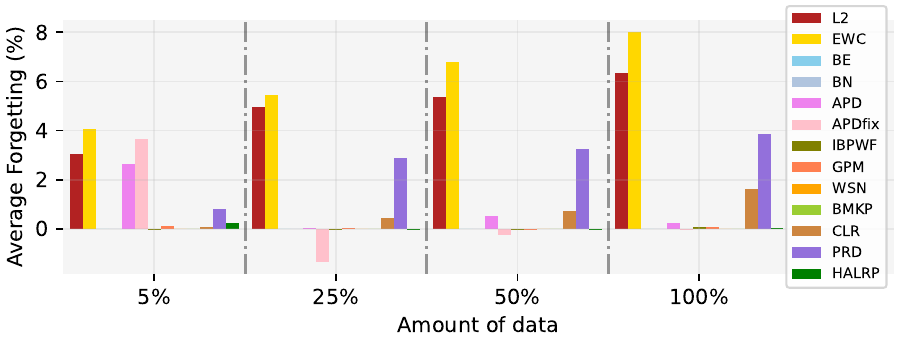}
         \caption{Forgetting on 20 tasks in CIFAR100-SuperClass}
         \label{fig:forgetting_c100super}
     \end{subfigure}
    \centering
     \begin{subfigure}[b]{0.9\linewidth}
         \centering
         \includegraphics[width=\linewidth]{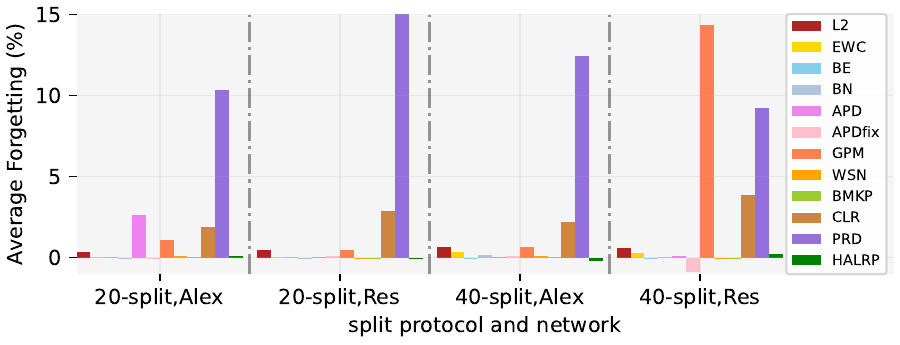}
         \caption{Forgetting on TinyImageNet. ``Alex'' for AlexNet and ``Res'' for ResNet18 }
         \label{fig:forgetting_tinyimagenet}
     \end{subfigure}
    \caption{Average Forgetting Statistics}
    \label{fig:comparison_forgetting}
\end{figure}

\subsection{Model Increment Analysis}
Another potential issue in continual learning is the model size growth as the number of tasks increases. To compare increased model capacities among different methods, we visualize \emph{the ratio of increased parameters} w.r.t to base model on the sequential 20 tasks of CIFAR100-SuperClass dataset in Fig.~\ref{fig:model_incrase}. We can observe that our proposed HALRP can better control the model capacity increment. After the first four tasks, the model parameters increased by around $19\%$ and grow slowly during the last sixteen tasks (finally $\leq28\%$). In contrast, the model parameters of APD grew quickly by around $40\%$ in the first two or three tasks and kept a high ratio during the following tasks. 
\begin{figure}
    \centering
     \begin{subfigure}[b]{0.48\linewidth}
         \centering
         \includegraphics[width=\linewidth]{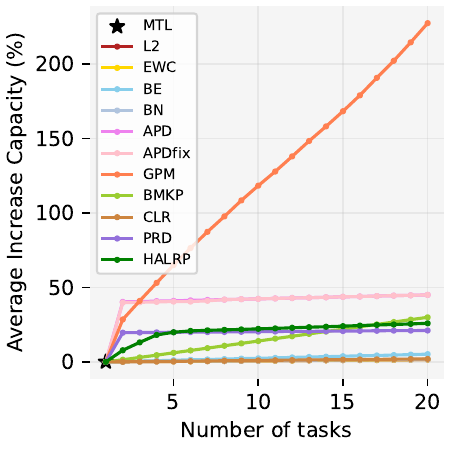}
         \caption{}
         \label{fig:model_incrase}
     \end{subfigure}
    \centering
     \begin{subfigure}[b]{0.48\linewidth}
         \centering
         \includegraphics[width=\linewidth]{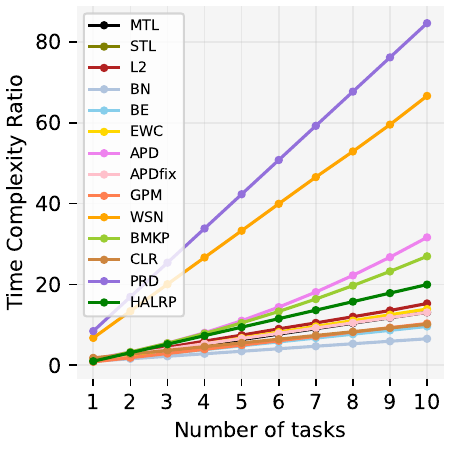}
         \caption{}
         \label{fig:time_complexity}
     \end{subfigure}
    \caption{(a) Average Capacity Increment ratio on CIFAR100-SuperClass w.r.t. the base model. (b) Average Time Complexity Ratio on PMNIST.}
    \label{fig:forgetting_time}
\end{figure}

\subsection{Computational Efficiency}
We visualize the average time complexity ratio on the ten tasks of the PMNIST dataset in Fig.~\ref{fig:time_complexity}. The time complexity ratio is computed by dividing the accumulated time across the tasks w.r.t. the time cost on the first task of single-task learning. We observe that our proposed method is also computationally efficient with limited time consumption compared to most baselines, with a better trade-off between performance and efficiency. Especially, our method needs less training time compared to WSN, BMKP and APD. 

The major reason that should account for the inefficiency of WSN is that this method needs to optimize a weight score for each model parameter and then generate binary masks by locating a certain sparsity quantile among the weight scores in each layer, which is very time-consuming.
Especially, we can observe a non-linear increase in behavior for BMKP and APD in Fig.~\ref{fig:time_complexity}. In BMKP, the pattern basis of the core knowledge space increases along the sequential tasks and then the knowledge projection between two memory levels will cost more time as the tasks pass. 
A severe drawback of APD method is that when the task $t$ arrives for learning, it needs to recover the parameters $\theta^{\star}_i$ for all previous tasks $i=0,1,...,t-1$ and then apply the regularization $\Sigma_{i=0}^{t-1}\|\theta_t - \theta^{\star}_i\|^2_2$ between the weights of current task (i.e., $\theta_t$) and each previous task (i.e., $\theta^{\star}_i$). This defect becomes more time-consuming as more tasks arrive. Besides, an extra $k$-means clustering process is needed for the hierarchical knowledge consolidation among tasks in APD. The difference of computational efficiency becomes more significant on the challenging scenarios with more tasks and complex network architectures. As for PRD, it is inefficient as this method converges slower than other methods, due to the reason that it adopted supervised contrastive loss instead of cross-entropy loss.

Apart from the time efficiency, we also tracked the GPU memory usage during the training. The results showed that our method only requires limited GPU memory for training compared to other state-of-the-art methods. See Appendix E.3 for more details.

\subsection{Ablation Studies}
In this part, we provide ablation studies about some hyperparameters in our method. 

\textbf{Ablation study about the rank selection}
We report the additional studies to evaluate the impacts of hyperparameter selection on the rank selection process.
We conduct the following ablations: (1) varying the warm-up epochs $n_r$, (2) varying the loss approximation rate $\alpha$, (3) \textbf{LRP}: omitting the importance estimation (step 6 of Algorithm~\ref{Alg.HALRP}), and (4) \textbf{Random LRP}: replacing step 6 of Algorithm~\ref{Alg.HALRP} with a random decomposition. 
The relevant results on CIFAR100-Splits are depicted in Table~\ref{Tab.Ablation}.
We can observe that the performance will not significantly rise as $n_r$ increases, which indicates that only one epoch is enough for warm-up training. When decreasing the approximate rate~$\alpha$, the accuracy will decay. Furthermore, if we apply the random decomposition of the model, the accuracy will sharply drop, indicating the necessity of our Hessian aware-decomposition procedure. 

\textbf{Effect of the regularization coefficients $\lambda_0$ and $\lambda_1$} To investigate the effects of the regularization coefficients introduced in Eq.~\ref{Eq:loss_reg}, we further conducted experiments on two datasets: CIFAR100-SuperClass and CIFAR100-Splits. We adopted values from $\{0, 1e-6, 1e-5, 1e-4, 1e-3\}$ for $\lambda_0$ and $\{5e-4, 1e-4, 5e-5\}$ for $\lambda_1$. For each combination, we reported MOPD/AOPD for task order robustness as well as the average accuracy on the above two datasets, and the results are illustrated in Fig.~\ref{fig:ablation-lambda}. We can see that the regularization coefficients $\lambda_0$ and $\lambda_1$ have potential effects on the final accuracy and task order robustness. In our experiments, we choose the optimal hyperparameters by validation. Furthermore, some selected hyperparameters are also applied to other baselines to realize fair comparisons, e.g., we set $L_2$-regularizer coefficient $\lambda_1=1e-4$ for all other methods on CIFAR100-SuperClass and CIFAR100-Splits. See Appendix F.3 for more details.

\begin{table}[]
\centering
\resizebox{0.5\linewidth}{!}{\Huge
\begin{tabular}{@{}l|cc@{}}
\toprule
Ablation                                        & Acc.$\uparrow$          &  Size$\downarrow$  \\ \midrule
{$\alpha$: 0.9, $n_r$: 1}  & $68.39\pm0.13$ & 0.234 \\
{$\alpha$: 0.9, $n_r$: 2}  & $67.93\pm0.19$ & 0.234 \\
{$\alpha$: 0.9, $n_r$: 3}  & $67.38\pm0.12$ & 0.234 \\
{$\alpha$: 0.9, $n_r$: 4}  & $67.04\pm0.08$ & 0.234 \\
{$\alpha$: 0.95, $n_r$: 1} & $68.09\pm0.15$ & 0.234 \\
{$\alpha$: 0.75, $n_r$: 1} & $68.26\pm0.11$ & 0.217 \\
{$\alpha$: 0.60, $n_r$: 1} & $66.87\pm0.23$ & 0.165 \\
{$\alpha$: 0.45, $n_r$: 1} & $66.08\pm0.14$ & 0.100 \\
\midrule
 LRP. & $67.49\pm$ 0.13 & 0.234  \\
Random LRP.      & $61.76\pm0.18$ & 0.235 \\ \bottomrule
\end{tabular}}
\caption{Ablation studies. Acc.$\uparrow$ refers to the accuracy; Size refers to the relative increment size.}
\label{Tab.Ablation}
\end{table}

\begin{figure}[t]
    \centering
     \begin{subfigure}[b]{\linewidth}
         \centering
         \includegraphics[width=\linewidth]{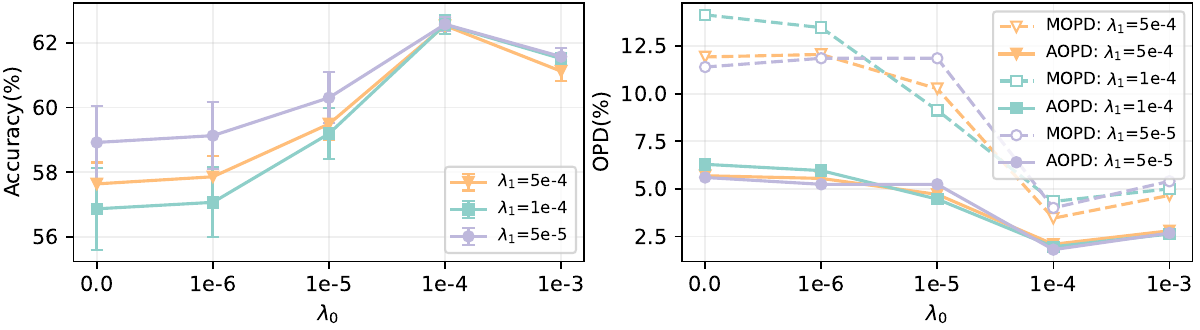}
         \caption{Accuracy$\uparrow$ (left) and MOPD$\downarrow$/AOPD$\downarrow$ (right) on CIFAR100-SuperClass (20 tasks)}
         \label{fig:ablation-lambda-c100super}
     \end{subfigure}
         \centering
     \begin{subfigure}[b]{\linewidth}
         \centering
         \includegraphics[width=\linewidth]{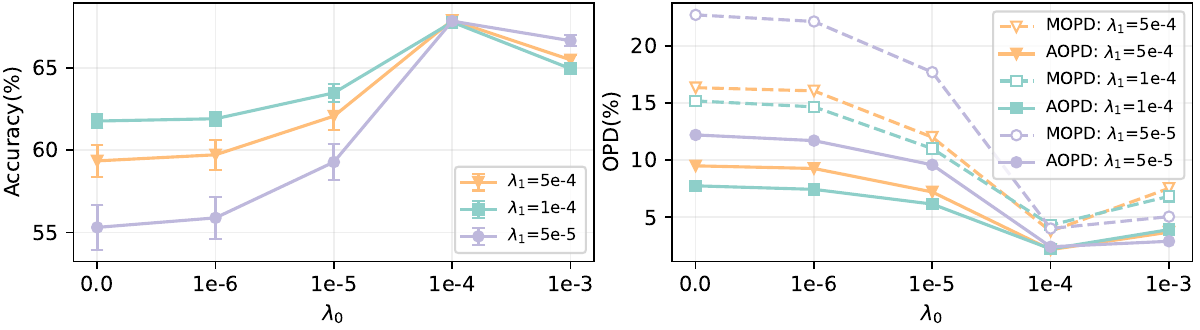}
        \caption{Accuracy$\uparrow$ (left) and MOPD$\downarrow$/AOPD$\downarrow$ (right) on CIFAR100-Splits (10 tasks)}
         \label{fig:ablation-lambda-c100splits}
     \end{subfigure}
    \caption{Effect of regularization coefficients $\lambda_0$ and $\lambda_1$.}
    \label{fig:ablation-lambda}
\end{figure}

\section{Conclusion}
In this work, we propose a low-rank perturbation method for continual learning. Specifically, we approximate the task-adaptive parameters with low-rank decomposition by formulating the model transition along the sequential tasks with parameter transformations. We theoretically show the quantitative relationship between the Hessian and the proposed low-rank approximation, which leads to a novel Hessian-aware framework that enables the model to automatically select ranks by the relevant importance of the perturbation to the model's performance. The extensive experimental results show that our proposed method performs better on the robustness of different task orders and the ability to address catastrophic forgetting issues.

\section*{Acknowledgments}
 The work of Fan Zhou, Rui Wang and Shichun Yang is supported by the National Natural Science Foundation of China, Young Scientist Fund (52302487). Jiaqi Li, Yuanhao Lai, Sabyasachi Sahoo, Charles X. Ling, Boyu Wang and Christian Gagn\'e are supported in part by the Science and the Natural Sciences and Engineering Research Council of Canada (NSERC), in part by the NSERC Discovery Grants Program.

\bibliographystyle{IEEEtran}
\bibliography{references}
\clearpage
\appendices

\section{Background on Low-Rank Factorization for Matrix}
\label{Sect.BackgroundSVD}

Here, we provide more background knowledge about the low-rank matrix factorization. 

In this work, we leverage the Eckart–Young–Mirsky theorem~\cite{eckart1936approximation} with the Frobenius norm. In this part, we provide background knowledge for the self-cohesion of the paper.

Denote by ${\B \in \mathbb {R} ^{m\times n}}$ a real (possibly rectangular) matrix. Suppose that $\B =\U\bSigma \V^{\top }$ is the singular value decomposition (SVD) of $\B$, then, we can claim that the best rank $k$ approximation ($k\leq \min\{m,n\}$) to $\B$ under the Frobenius norm $\|\cdot \|_{F}$ is given by
${\B_{k}=\sum _{i=1}^{k}\sigma _{i} \mathbf{u}_{i}\mathbf{v}_{i}^{\top }}$, where  $\mathbf{u}_{i}$ and $\mathbf{v}_{i}$ denote the $i^{th}$ column of $\U$ and $\V$, respectively. Then,

$$\|\B-\B_{k}\|_{F}^{2}=\Big\|\sum _{i=k+1}^{n}\sigma _{i}\mathbf{u}_{i}\mathbf{v}_{i}^{\top }\Big\|_{F}^{2}=\sum _{i=k+1}^{n}\sigma _{i}^{2}$$

Thus, we need to show that if $\A_{k} =XY^{\top}$ where $X$ and $Y$ has $k$ columns

$$\|\B-\B_{k}\|_{F}^{2}=\sum _{i=k+1}^{n}\sigma_{i}^{2}\leq \|\B-\A_{k}\|_{F}^{2}$$.

By the triangle inequality, if $\B=\B'+\B''$ then $\sigma _{1}(\B)\leq \sigma _{1}(\B')+\sigma _{1}(\B'')$. Denote by $\B'_{k}$ and $\B''_{k}$ the rank $k$ approximation to $\B'$ and $\B''$ by the SVD method, respectively. Then, for any $i,j\geq 1$,

\begin{equation*}
	\begin{split}
		\sigma _{i}(\B')+\sigma _{j}(\B'')&=\sigma _{1}(\B'-\B'_{i-1})+\sigma _{1}(\B''-\B''_{j-1})\\
		&\geq \sigma _{1}(\B-\B'_{i-1}-\B''_{j-1})\\&\geq \sigma _{1}(\B-\B_{i+j-2})\qquad \\&=\sigma _{i+j-1}(\B)
	\end{split}
\end{equation*}
where the last inequality comes from the fact that ${\rm {rank}}(\B'_{i-1}+\B''_{j-1})\leq {\rm {rank\,}}(\B_{i+j-2})$. 

Since $\sigma _{k+1}(\B_{k})=0$, when $\B'=\B-\A_{k}$ and $\B''= \A_{k}$ we conclude that for $ i\geq 1,j=k+1$

$$\sigma_{i}(\B-\A_{k})\geq \sigma_{k+i}(\B)$$
Therefore,
\begin{equation*}
	\begin{split}
		\|\B-\A_{k}\|_{F}^{2}=&\sum _{i=1}^{n}\sigma _{i}(\B-\A_{k})^{2}\geq  \sum _{i=k+1}^{n}\sigma _{i}(\B)^{2}=\|\B-\B_{k}\|_{F}^{2}
	\end{split}
\end{equation*}
Thus, we can get Eq.~1 displayed in the paper.

\section{Proof to Theorem 1 and Discussion}
\label{sec.proofTheorem}

\begin{theoremnew}
	Assume that a neural network of $L$ layers with vectorized weights 
	$(\bomega^{\star}_1, \dots, \bomega^{\star}_L)$ that have converged to local optima,
	such that the first and second order
	optimality conditions are satisfied, i.e., the gradient is zero, and the Hessian is positive semi-definite.
	Suppose a perturbation $\Delta \bomega^{\star}_1$ applied to the first layer
	weights, then we have the loss change
	\begin{equation}
		\begin{split}
			|\mathcal{L}(\bomega^{\star}_1 -\Delta \bomega^{\star}_1,& \dots, \bomega^{\star}_L) - 
			\mathcal{L}(\bomega^{\star}_1, \dots, \bomega^{\star}_L)| \\ 
			&\leq \frac{1}{2}\|\bH_1\|_F \cdot \|\Delta \bomega^{\star}_1\|^2_F
			+o(\|\Delta \bomega^{\star}_1\|^2_F),
		\end{split}
	\end{equation}
	where $\bH_1=\nabla^2\mathcal{L}(\bomega^{\star}_1)$ is the Hessian matrix at only the variables of the first layer weights.
\end{theoremnew}
\begin{proof}
	Denote the gradient and Hessian of the first layer $\bomega^{\star}_1$ as $\bg_1$ and $\bH_1$.
	Through Taylor’s expansion, we have
	\begin{equation*}
		\begin{split}
			\mathcal{L}(\bomega^{\star}_1 & -\Delta \bomega^{\star}_1, \dots, \bomega^{\star}_L) - 
			\mathcal{L}(\bomega^{\star}_1, \dots, \bomega^{\star}_L)\\  
			& = -\bg_1^T \Delta \bomega^{\star}_1 +
			\frac{1}{2}\Delta\bomega^{\star T}_1 \bH_1 \Delta\bomega^{\star}_1
			+o(\|\Delta \bomega^{\star}_1\|^2_F).
		\end{split}
	\end{equation*}
	Using the fact that the gradient is zero at the local optimum $\bomega^{\star}_1$  as well as the sub-additive and sub-multiplicative properties of Frobenius norm, 
	we have,
	\begin{equation*}
		\begin{split}
			|\mathcal{L}(\bomega^{\star}_1 & -\Delta \bomega^{\star}_1, \dots, \bomega^{\star}_L) - 
			\mathcal{L}(\bomega^{\star}_1, \dots, \bomega^{\star}_L)|\\  
			& = |\frac{1}{2}\Delta\bomega^{\star T}_1 \bH_1 \Delta\bomega^{\star}_1 +o(\|\Delta \bomega^{\star}_1\|^2_F) |\\
			& \leq \|\frac{1}{2}\Delta\bomega^{\star T}_1 \bH_1 \Delta\bomega^{\star}_1\|_F
			+ o(\|\Delta \bomega^{\star}_1\|^2_F) \\
			& \leq \frac{1}{2}\|\bH_1\|_F \cdot \|\Delta\bomega^{\star}_1\|^2_F
			+ o(\|\Delta \bomega^{\star}_1\|^2_F).
		\end{split}
	\end{equation*}
	
	Thus, we conclude the proof. 
	
\end{proof}

Algorithm 1 implies that the Hessian information
can be used to quantitatively measure the influences of
low-rank perturbation on the model’s empirical losses. In practice, we can approximate the Hessian by
the negative empirical Fisher information~\cite{kunstner2019limitations}. This enables a dynamic scheme for the trade-off between the approximation error and computational efficiency. For a given loss approximation rate, the model can
automatically select the rank for all the layers.

\section{Solving $\R^{\text{free}}, \boldS^{\text{free}}, \B^{\text{free}}$}
\label{Sect:appendix-RSB}
In Section 4.1, we described that when a new task comes, we first train the model on it for several epochs. With this step, we can obtain a rough parameter estimation $\W^{\text{free}}$ of this new task, which indicates the freely-trained weights without any constraints. Then, we can obtain $\R^{\text{free}}, \boldS^{\text{free}}, \B^{\text{free}}$ by solving the following problem:
\begin{align}
	&\R^{\text{free}} = \argmin_{\R} \|\W^{\text{free}} -  \R \W^{\text{base}}\|_{F}^2 \\
	&\boldS^{\text{free}} = \argmin_{\boldS}  \|\W^{\text{free}}  -  \R^{\text{free}} \W^{\text{base}} \boldS \|_{F}^2\\
	&\B^{\text{free}} = \W^{\text{free}} - \R^{\text{free}} \W^{\text{base}} \boldS^{\text{free}}
\end{align}
Simply, $\|\W^{\text{free}} -  \R \W^{\text{base}}\|_{F}^2= \sum_{i=1}^I \sum_{j=1}^J (w^{\text{free}}_{ji} -r_jw^{\text{base}}_{ji})^2$, where $w^{\text{free}}_{ji}$ and $w^{\text{base}}_{ji}$ are the elements of $j$-th row and $i$-th column of {\small$\W^{\text{free}}$} and {\small$\W^{\text{base}}$}, respectively. By taking the derivative of the above expression w.r.t. $r_j$ and let it equals $0$, we can obtain:
\begin{equation}
	r_j^{\text{free}} = \frac{\sum^{I}_{i=1} w_{ji}^{\text{free}}w_{ji}^{\text{base}}}{\sum^{I}_{i=1} (w_{ji}^{\text{base}})^2} 
\end{equation}
Similarly, $\|\W^{\text{free}}  -  \R^{\text{free}} \W^{\text{base}} \boldS \|_{F}^2=\sum^{I}_{i=1}\sum^{J}_{j=1}( w_{ji}^{\text{free}} - r_{j}^{\text{free}} w_{ji}^{\text{base}} s_{i} )^2$. By taking the derivative w.r.t. $s_i$ and let it equals $0$, we can obtain:
\begin{equation}
	s_i^{\text{free}} = \frac{\sum^{J}_{j=1} r_{j}^{\text{free}} w_{ji}^{\text{free}}w_{ji}^{\text{base}}}{\sum^{J}_{j=1} (r_{j}^{\text{free}}w_{ji}^{\text{base}})^2}
\end{equation}
Then, we can calculate $\B^{\text{free}}$ by the third equation using the obtained $\R^{\text{free}}$ and $\boldS^{\text{free}}$. 

Finally, the low rank approximation can be applied on $\W^{\text{free}}$ with \texttt{SVD}:
\begin{equation}
	\W^{{\text{free}}} \approx \W^{(k){\text{free}}} = \R^{\text{free}} \W^{\text{base}} \boldS^{\text{free}} +
	\B^{(k){\text{free}}},
\end{equation}
where {\small$\B^{(k){\text{free}}} = \U^{(k){\text{free}}} \bSigma^{(k){\text{free}}} (\V^{(k){\text{free}}})^\top $}.

Note that the number of parameters in the original weights $\W^{\text{free}} \in \mathbb{R}^{J\times I}$ is $\text{size}\{\W^{\text{free}}\}=JI$. After this approximation, we only need to store the weights $\R^{\text{free}}$, $\boldS^{\text{free}}$, $\U^{(k){\text{free}}}$, $\bSigma^{(k){\text{free}}}$, $\V^{(k){\text{free}}}$ for a new task, whose parameter number is $\text{size}\{\R^{\text{free}}, \boldS^{\text{free}}, \U^{(k){\text{free}}}, \bSigma^{(k){\text{free}}}, \V^{(k){\text{free}}}\}=J+I+kJ+k+kI=(J+I)(k+1)+k$. Thus, the incremental ratio is $\rho = \frac{(J+I)(k+1)+k}{JI} \ll 1$ in practice.

\section{Discussion on the Fine-tuning Objective on the New Task}
\label{Sect:discussionRegularization}
In the proposed algorithm, we finally fine-tune the model on the new task with Eq.~12 in the manuscript, which is re-illustrated below,
\begin{equation}
	\min_{\W_t} \calL(\W_t; \mathcal{D}_t) + \calL_{\text{reg}}(\W_t)
\end{equation}
where 
\begin{equation}
	\begin{aligned}
		&\calL_{\text{reg}}(\W_t) =\sum_{l}\Big[
		\lambda_0 \underbrace{(  \| \U_{i}^{(k_l) \text{free}}  \| + \| \V_{i}^{(k_l) \text{free}}   \| )}_{L_{1}-\mathrm{regularization}} \\ 
		&+ \lambda_1 \underbrace{( \| \R_i^{\text{free}} \|_2^2 +  \| \boldS_i^{\text{free}} \|_2^2 + \| \U_{i}^{(k_l) \text{free}} \|_2^2 + \|\V_{i}^{(k_l) \text{free}} \|_2^2)}_{L_{2}-\mathrm{ regularization}}
		\Big]
	\end{aligned}
\end{equation}

The fine-tuning objective mainly consists of three parts, the general cross-entropy loss on the new task $\mathcal{T}_t$, a $L_1$ regularization term on $\| \U_{i}^{(k_l) \text{free}}  \| + \| \V_{i}^{(k_l) \text{free}}  \|$, and the $L_2$ regularization terms on $\| \R_i^{\text{free}} \|_2^2 +  \| \boldS_i^{\text{free}} \|_2^2$ and $\| \U_{i}^{(k_l) \text{free}} \|_2^2 + \|\V_{i}^{(k_l) \text{free}} \|_2^2$, respectively. 

As mentioned in the paper, we apply the $L_1$ regularization on $\U_{i}^{(k_l) \text{free}}$ and $ \V_{i}^{(k_l) \text{free}}$. Besides, as discussed in Section 4.4 of the paper, we also apply to prune $\U_{i}^{\text{free}}$ and $ \V_{i}^{(k_l) \text{free}}$, which further encouraged the sparsity.


Our method keeps $\mathbf{W}^{\text{base}}$ as unchanged for new tasks for knowledge transfer, while the $\R_i^{\text{free}}$, $\boldS_i^{\text{free}}$, $\U_{i}^{(k) \text{free}}$ and $\V_{i}^{(k)\text{free}}$ are left as task-adaptive parameters. Since $\R_i^{\text{free}}$ and $\boldS_i^{\text{free}}$ are diagonal matrices, plus the fact that $\U_{i}^{(k) \text{free}}$ and $ \V_{i}^{(k) \text{free}}$ are sparse, thus the model only needs to learn a small number of parameters. Unlike some regularization methods (\emph{e.g.}~EWC~\cite{kirkpatrick2017overcoming}), which constrain the gradient update and require to re-train a lot of parameters, our method can update the model more efficiently. In addition, the empirical results reported in the paper show that our method can also achieve better performance with less time and memory cost. 

\section{Additional Experimental Results}
\label{Sect.AdditionalExp}

\begin{table}[t]
	\centering
	\resizebox{0.65\linewidth}{!}{
		\begin{tabular}{@{}l|ccc@{}}
			\toprule
			& \multicolumn{3}{c}{P-MNIST}      \\ \cmidrule(l){2-4} 
			& Acc.$\uparrow$              & MOPD$\downarrow$  & AOPD$\downarrow$ \\ \cmidrule(l){2-4} 
			IBWPF & 78.12 $\pm$ 0.83  & 12.69 & 6.65 \\
			HALRP & 98.10 $\pm$ 0.03  & 0.47  & 0.24 \\ \midrule
			& \multicolumn{3}{c}{Five-dataset} \\ \cmidrule(l){2-4} 
			& Acc.$\uparrow$              & MOPD$\downarrow$  & AOPD$\downarrow$ \\ \cmidrule(l){2-4} 
			IBWPF & 84.62 $\pm$ 0.36  & 5.06  & 1.72 \\
			HALRP & 88.81 $\pm$ 0.31  & 4.28  & 1.31 \\ \bottomrule
	\end{tabular}}
	\caption{Comparison with low-rank factorization method}
	\label{Tab:ComparsionWithIBPWF}
\end{table}

\begin{figure*}[t]
	\centering
	\begin{subfigure}[b]{0.24\linewidth}
		\centering
		\includegraphics[width=\linewidth]{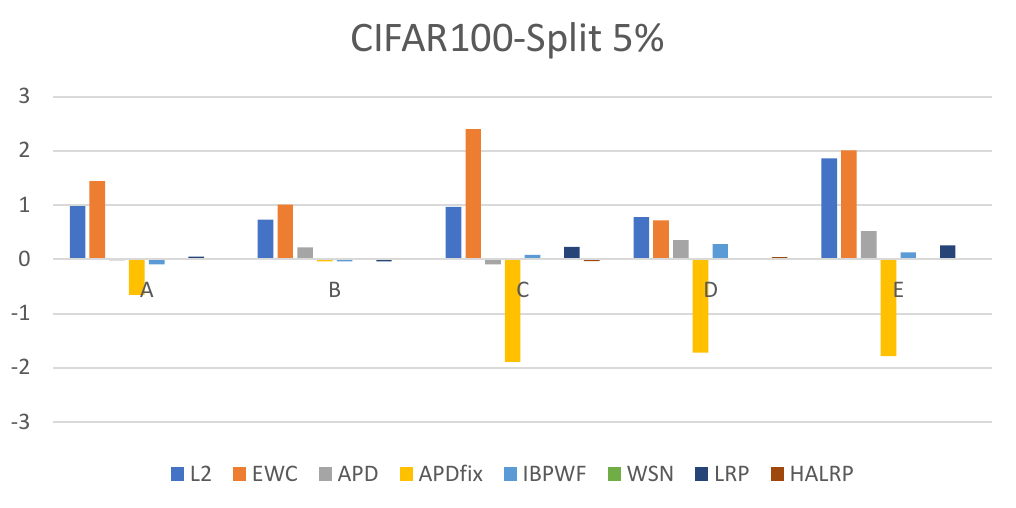}
		\caption{5\% of training data}
		\label{fig:forgetting_split5}
	\end{subfigure}
	\centering
	\begin{subfigure}[b]{0.24\linewidth}
		\centering
		\includegraphics[width=\linewidth]{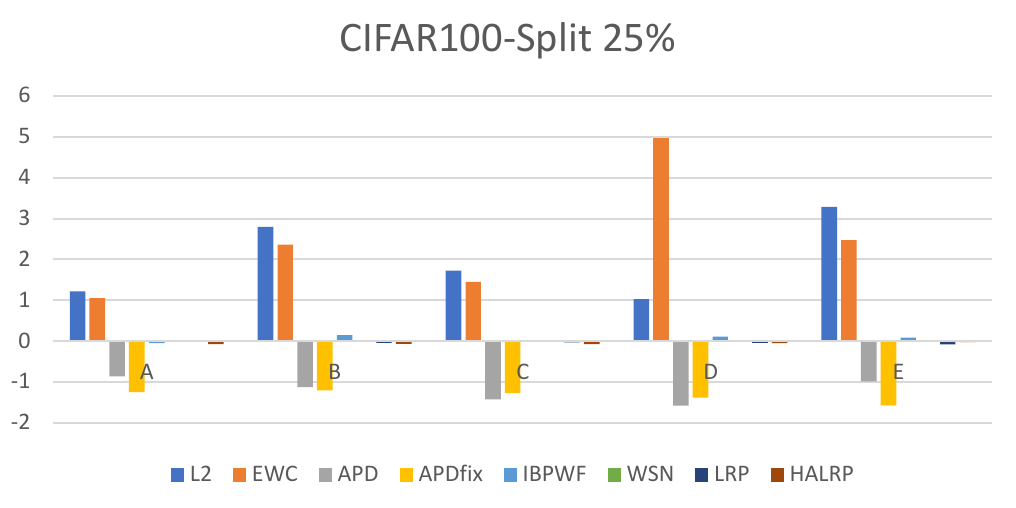}
		\caption{25\% of training data}
		\label{fig:forgetting_split25}
	\end{subfigure}
	\begin{subfigure}[b]{0.24\linewidth}
		\centering
		\includegraphics[width=\linewidth]{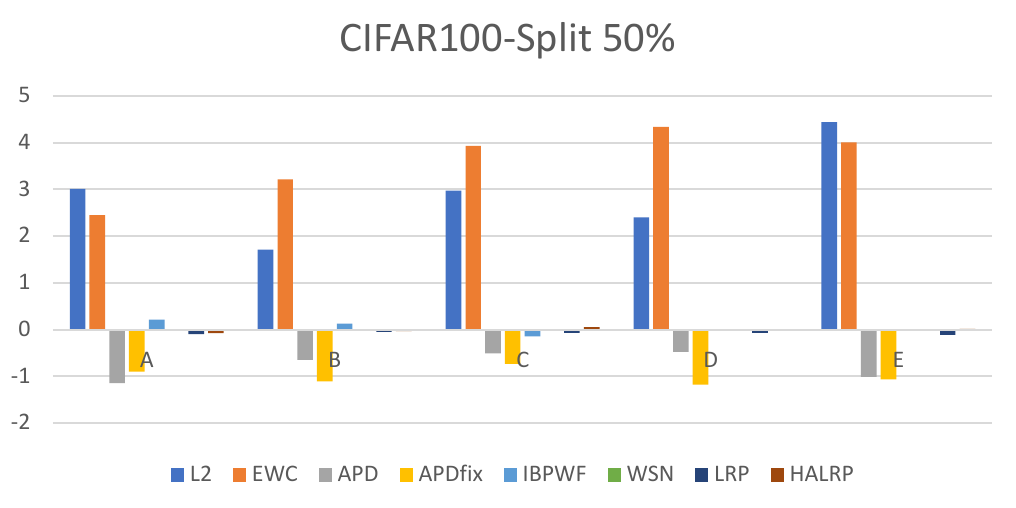}
		\caption{50\% of training data}
		\label{fig:forgetting_split50}
	\end{subfigure}
	\begin{subfigure}[b]{0.24\linewidth}
		\centering
		\includegraphics[width=\linewidth]{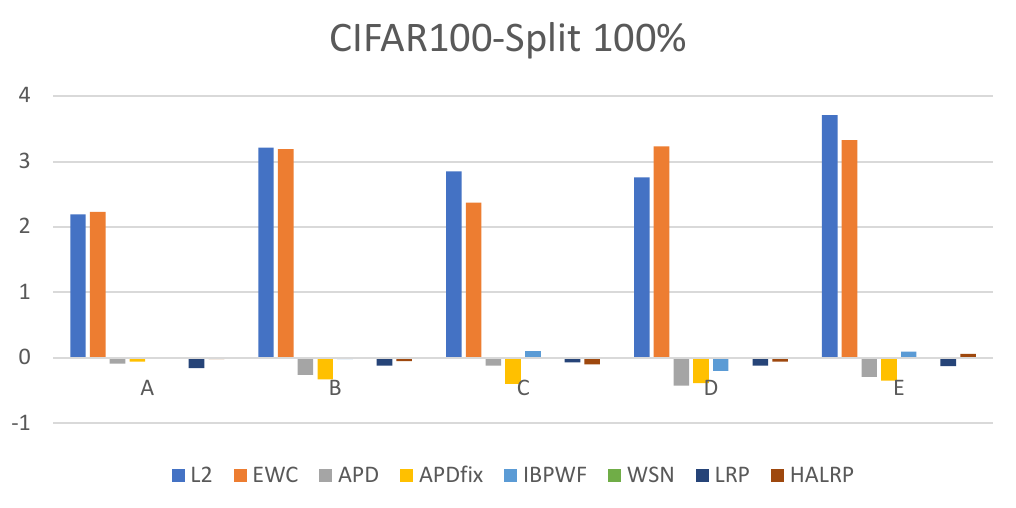}
		\caption{100\% of training data}
		\label{fig:forgetting_split100}
	\end{subfigure}
	\caption{Forgetting comparison on CIFAR100-Split with different task orders (A-E) under different amounts of training data.}
	\label{fig:comparison_forgetting_C100SPLIT}
\end{figure*}

\begin{figure*}[t]
	\centering
	\begin{subfigure}[b]{0.24\linewidth}
		\centering
		\includegraphics[width=\linewidth]{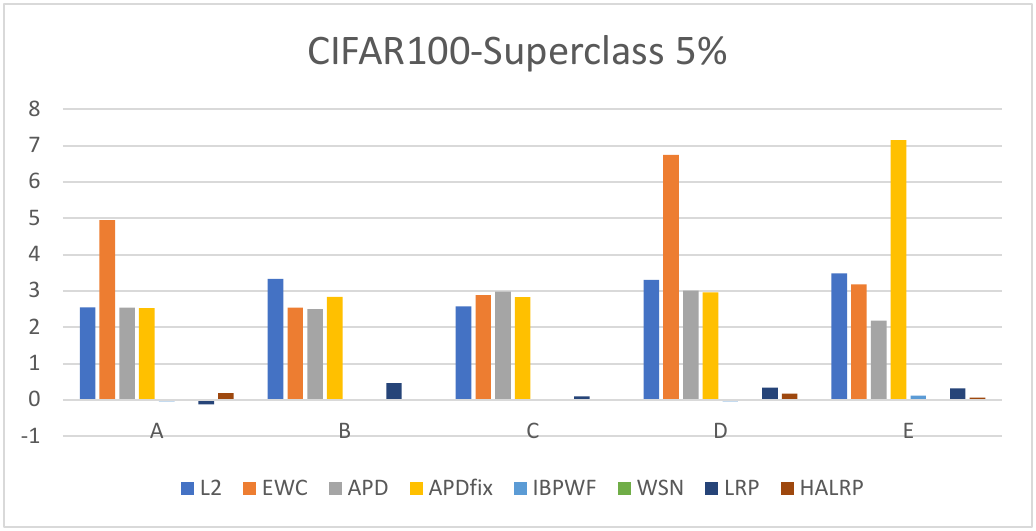}
		\caption{5\% of training data}
		\label{fig:forgetting_super5}
	\end{subfigure}
	\centering
	\begin{subfigure}[b]{0.24\linewidth}
		\centering
		\includegraphics[width=\linewidth]{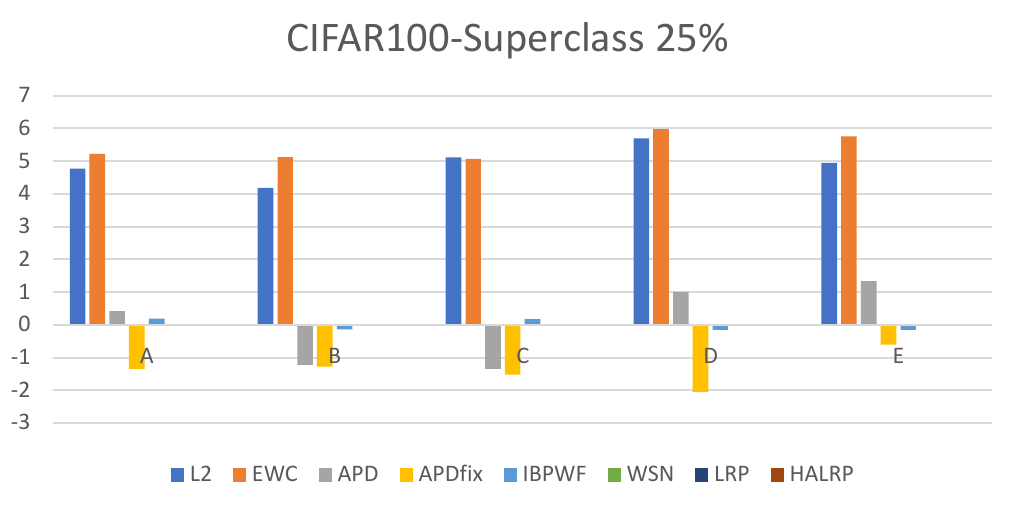}
		\caption{25\% of training data}
		\label{fig:forgetting_super25}
	\end{subfigure}
	\begin{subfigure}[b]{0.24\linewidth}
		\centering
		\includegraphics[width=\linewidth]{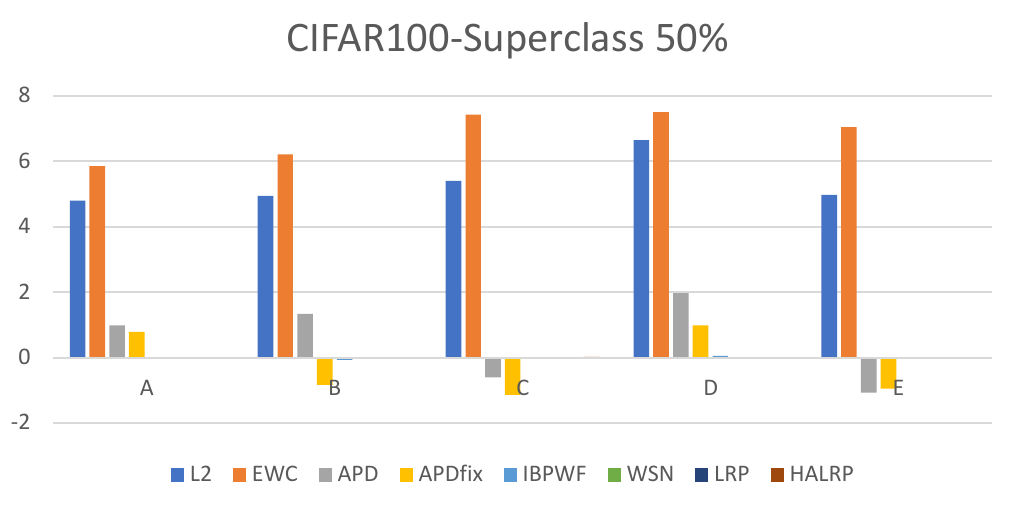}
		\caption{50\% of training data}
		\label{fig:forgetting_super50}
	\end{subfigure}
	\begin{subfigure}[b]{0.24\linewidth}
		\centering
		\includegraphics[width=\linewidth]{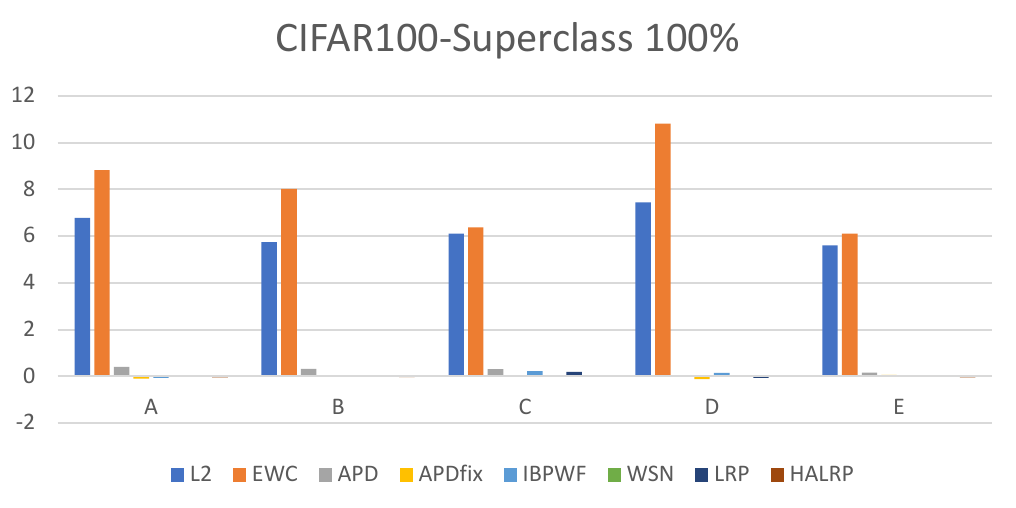}
		\caption{100\% of training data}
		\label{fig:forgetting_super100}
	\end{subfigure}
	\caption{Forgetting on CIFAR100-SuperClass with different task orders (A-E) under different amounts of training data.}
	\label{fig:comparison_forgetting_C100_SUPER}
\end{figure*}

\subsection{Comparing with other low-rank methods}
\label{app-exp:low-rank}

Our work also shares some similarities with the low-rank-decomposition-based method IBPWF~\cite{mehta2021continual}. As discussed in~\cite{hyder2022incremental} and in our paper (Sec. 1), this method requires larger ranks to accept higher accuracy. Furthermore, as pointed out by~\cite{verma2021efficient}, IBPWF leverages Bayesian non-parametric to let the data dictate expansion, but the benchmarks considered in Bayesian methods have been limited to smaller datasets, like MNIST and CIFAR-10. In addition to the experimental comparisons with IBPWF reported in the paper, we further report more results in Table~\ref{Tab:ComparsionWithIBPWF}. We can observe that our method outperforms IBPWF with a large gap in terms of Accuracy and MOPD$\downarrow$ \& AOPD$\downarrow$. Furthermore, we observe that IBPWF cannot perform well on P-MNIST, which is $10\times$ the number of data of MNIST. The gap is  
consistent with observations in~\cite{verma2021efficient}. 

\subsection{Performance on Alleviate Forgetting on Different Task Orders}
\label{app-exp:forgetting}
In the main paper, we report the forgetting performance on the average of five task orders (A-E) on CIFAR100 Splits and SuperClass. In this part, we provide the performance on forgetting under each task order in Fig.~\ref{fig:comparison_forgetting_C100SPLIT} and Fig.~\ref{fig:comparison_forgetting_C100_SUPER}, under different amounts (i.e., $5\% \sim 100\%$) of training data from.

\subsection{Analysis of Memory Overhead}\label{sec:append:memory}
\begin{figure*}[h]
	\centering
	\begin{subfigure}[b]{\textwidth}
		\centering
		\includegraphics[width=\textwidth]{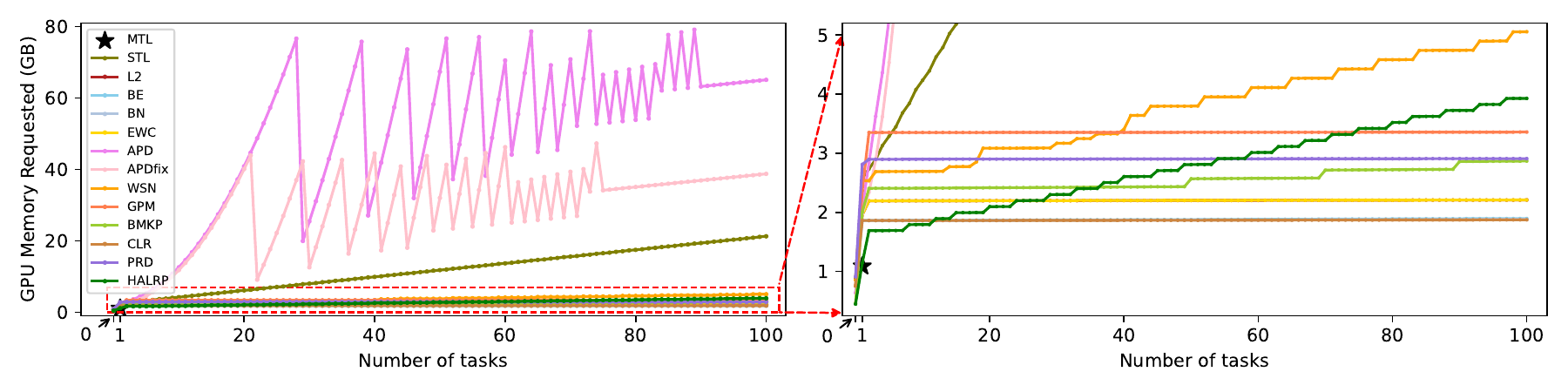}
		\caption{GPU memory usage for Omniglot-Rotation with LeNet}
		\label{Fig:memory-overhead-omni}
	\end{subfigure}
	
	\begin{subfigure}[b]{\textwidth}
		\centering
		\includegraphics[width=\textwidth]{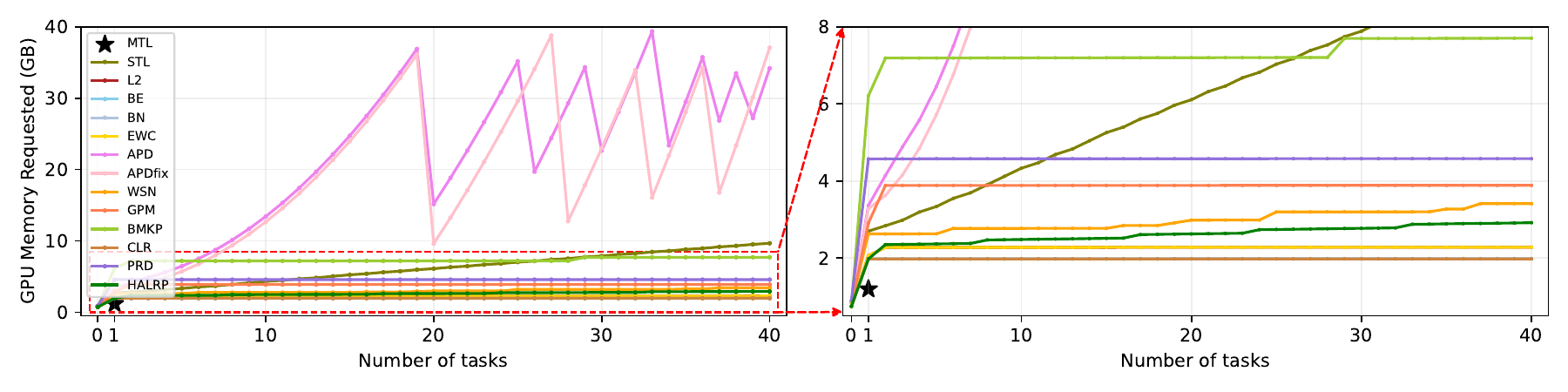}
		\caption{GPU memory usage for TinyImageNet 40-split with AlexNet}
		\label{Fig:memory-overhead-tiny40}
	\end{subfigure}
	\caption{Empirical statistics of GPU memory usage. The local zone with red rectangle in left is shown in the right plot.}
	\label{Fig:memory-overhead}
\end{figure*}
We also provided some quantitative results to support the memory efficiency of our proposed HALRP. To compare the GPU memory overhead among different methods, we visualize the amount of GPU memory requested by each method along the increase of task numbers during the training process. Specifically, we tracked the GPU memory usage of two groups of representative and challenging experimental scenarios: (1) Omniglot Rotation dataset with LeNet, which has 100 tasks in total; (2) TinyImageNet dataset with AlexNet, which has the largest number of parameters (i.e., \#parameters $\approx$ 62 million, FLOPS $\approx$ 724 million in AlexNet) and the second largest number of tasks (i.e., 40 tasks). To make fair comparisons, we adopted the same hyperparemeters (i.e., batch size, number of threads used in the data loader) for all the methods under each scenario, and moved all training-related operations (i.e., Singular Value Decomposition) onto the GPU devices. To monitor the GPU memory usage, we embedded some snippets with the Python package \texttt{GPUtil} (\href{https://pypi.org/project/GPUtil/}{https://pypi.org/project/GPUtil/}) into the training code. The amount of GPU memory requested during the training process is shown in the following Figure~\ref{Fig:memory-overhead}.
\\ \hspace*{\fill}\\
According to the visualization results in Figure~\ref{Fig:memory-overhead}, we can observe that our proposed HALRP is still memory efficient compared to other baseline methods under these two challenging scenarios. Specifically, under the scenario for Omniglot Rotation datasets with LeNet (Fig.~\ref{Fig:memory-overhead-omni}, we can see that the amount of GPU memory requested by HALRP increased along the tasks but finally not exceeded 4GB for these 100 tasks. In contrast, APD and APDfix needed more GPU memory during training (i.e., maximum 80GB for APD and 45GB for APDfix), making the related experiments hard to be reproduced unless on some specific GPUs like NVIDIA A100 80G. As for TinyImageNet dataset with AlexNet, our proposed HALRP only requested about 3GB GPU memory at the end of the 40-th task, which is much lower than APD and APDfix that needed at least 40GB memory during training, as well as BMKP that needs up to 8GB.
\\ \hspace*{\fill}\\
To summarize, our proposed method HALRP won't introduce heavy memory overhead. In the contrast, it is memory-efficient. The empirical results show that the requirement of training is easy to be satisfied and the scalability under the above challenging scenarios is easy to achieve.

\section{Experimental Details}
\label{Sect.ExpDetails}

\subsection{Dataset Preparation}

\textbf{CIFAR100 Splits/SuperClass}
We used the CIFAR100 Splits/SuperClass dataset following the evaluation protocol of~\cite{APD:YoonKYH20}, and follow the task order definition of~\cite{APD:YoonKYH20} to test the algorithms with five different task orders (A-E).

For the CIFAR100 Split, the task orders are defined as:
\begin{itemize}
	\item Order A: [0,1,2,3,4,5,6,7,8,9]
	\item Order B: [1, 7, 4, 5, 2, 0, 8, 6, 9, 3]
	\item Order C: [7, 0, 5, 1, 8, 4, 3, 6, 2, 9]
	\item Order D: [5, 8, 2, 9, 0, 4, 3, 7, 6, 1]
	\item Order E: [2, 9, 5, 4, 8, 0, 6, 1, 3, 7]
\end{itemize}

For the CIFAR100 SuperClass, the task orders are:
\begin{itemize}
	\item Order A: [0, 1, 2, 3, 4, 5, 6, 7, 8, 9, 10, 11, 12, 13, 14, 15, 16, 17, 18, 19]
	\item Order B: [15, 12, 5, 9, 7, 16, 18, 17, 1, 0, 3, 8, 11, 14, 10, 6, 2, 4, 13, 19]
	\item Order C: [17, 1, 19, 18, 12, 7, 6, 0, 11, 15, 10, 5, 13, 3, 9, 16, 4, 14, 2, 8]
	\item Order D: [11, 9, 6, 5, 12, 4, 0, 10, 13, 7, 14, 3, 15, 16, 8, 1, 2, 19, 18, 17]
	\item Order E: [6, 14, 0, 11, 12, 17, 13, 4, 9, 1, 7, 19, 8, 10, 3, 15, 18, 5, 2, 16]
\end{itemize}

Furthermore, as discussed in the paper, we demonstrate the performance when handling limited training data. In this regard, we randomly select $5\%$, $25\%$, $50\%$ training data from each task and report the corresponding accuracies.

\textbf{P-MNIST}
We follow~\cite{kang2022forget} to evaluate the algorithms' performance on the P-MNIST dataset. Each task of P-MNIST is a random permutation of the original MNIST pixel. We follow~\cite{kang2022forget,ebrahimi2019uncertainty} to generate the train/val/test splits and to create 10 sequential tasks using different permutations, and each task has 10 classes. We randomly generate five different task orders with five different seeds. 
\begin{itemize}
	\item seed 0: [6, 1, 9, 2, 7, 5, 8, 0, 3, 4]
	\item seed 1: [2, 9, 6, 4, 0, 3, 1, 7, 8, 5]
	\item seed 2: [4, 1, 5, 0, 7, 2, 3, 6, 9, 8]
	\item seed 3: [5, 4, 1, 2, 9, 6, 7, 0, 3, 8]
	\item seed 4: [3, 8, 4, 9, 2, 6, 0, 1, 5, 7]
\end{itemize}

\textbf{Five dataset}
It uses a sequence of 5 different benchmarks including CIFAR10~\cite{krizhevsky2009learning}, MNIST~\cite{deng2012mnist}, notMNIST~\cite{bulatov2011notmnist}, FashionMNIST~\cite{xiao2017fashion} and SVHN~\cite{yuval2011reading}. Each benchmark contains 10 classes. We follow~\cite{kang2022forget} to generate the train/val/test splits and to create 10 sequential tasks using different permutations, and each task has 10 classes. We randomly generate five different task orders with five different seeds.
\begin{itemize}
	\item seed 0: [2, 0, 1, 3, 4]
	\item seed 1: [1, 0, 4, 2, 3]
	\item seed 2: [2, 4, 1, 3, 0]
	\item seed 3: [3, 4, 1, 0, 2]
	\item seed 4: [0, 3, 1, 4, 2]
\end{itemize}

\textbf{Omniglot-rotation}
We split this dataset~\cite{lake2015human} into 100 12-way classification tasks. We follow the train/val/test split of~\cite{kang2022forget}. The five task order adopted in our experiments are:
\begin{itemize}
	\item seed 0: [26, 86, 2, 55, 75, 93, 16, 73, 54, 95, 53, 92, 78, 13, 7, 30, 22, 24, 33, 8, 43, 62, 3, 71, 45, 48, 6, 99, 82, 76, 60, 80, 90, 68, 51, 27, 18, 56, 63, 74, 1, 61, 42, 41, 4, 15, 17, 40, 38, 5, 91, 59, 0, 34, 28, 50, 11, 35, 23, 52, 10, 31, 66, 57, 79, 85, 32, 84, 14, 89, 19, 29, 49, 97, 98, 69, 20, 94, 72, 77, 25, 37, 81, 46, 39, 65, 58, 12, 88, 70, 87, 36, 21, 83, 9, 96, 67, 64, 47, 44]
	\item seed 1: [80, 84, 33, 81, 93, 17, 36, 82, 69, 65, 92, 39, 56, 52, 51, 32, 31, 44, 78, 10, 2, 73, 97, 62, 19, 35, 94, 27, 46, 38, 67, 99, 54, 95, 88, 40, 48, 59, 23, 34, 86, 53, 77, 15, 83, 41, 45, 91, 26, 98, 43, 55, 24, 4, 58, 49, 21, 87, 3, 74, 30, 66, 70, 42, 47, 89, 8, 60, 0, 90, 57, 22, 61, 63, 7, 96, 13, 68, 85, 14, 29, 28, 11, 18, 20, 50, 25, 6, 71, 76, 1, 16, 64, 79, 5, 75, 9, 72, 12, 37]
	\item seed 2: [83, 30, 56, 24, 16, 23, 2, 27, 28, 13, 99, 92, 76, 14, 0, 21, 3, 29, 61, 79, 35, 11, 84, 44, 73, 5, 25, 77, 74, 62, 65, 1, 18, 48, 36, 78, 6, 89, 91, 10, 12, 53, 87, 54, 95, 32, 19, 26, 60, 55, 9, 96, 17, 59, 57, 41, 64, 45, 97, 8, 71, 94, 90, 98, 86, 80, 50, 52, 66, 88, 70, 46, 68, 69, 81, 58, 33, 38, 51, 42, 4, 67, 39, 37, 20, 31, 63, 47, 85, 93, 49, 34, 7, 75, 82, 43, 22, 72, 15, 40]
	\item seed 3: [93, 67, 6, 64, 96, 83, 98, 42, 25, 15, 77, 9, 71, 97, 34, 75, 82, 23, 59, 45, 73, 12, 8, 4, 79, 86, 17, 65, 47, 50, 30, 5, 13, 31, 88, 11, 58, 85, 32, 40, 16, 27, 35, 36, 92, 90, 78, 76, 68, 46, 53, 70, 80, 61, 18, 91, 57, 95, 54, 55, 28, 52, 84, 89, 49, 87, 37, 48, 33, 43, 7, 62, 99, 29, 69, 51, 1, 60, 63, 2, 66, 22, 81, 26, 14, 39, 44, 20, 38, 94, 10, 41, 74, 19, 21, 0, 72, 56, 3, 24]
	\item seed 4: [20, 10, 96, 16, 63, 24, 53, 97, 41, 47, 43, 2, 95, 26, 13, 37, 14, 29, 35, 54, 80, 4, 81, 76, 85, 60, 5, 70, 71, 19, 65, 62, 27, 75, 61, 78, 18, 88, 7, 39, 6, 77, 11, 59, 22, 94, 23, 12, 92, 25, 83, 48, 17, 68, 31, 34, 15, 51, 86, 82, 28, 64, 67, 33, 45, 42, 40, 32, 91, 74, 49, 8, 30, 99, 66, 56, 84, 73, 79, 21, 89, 0, 3, 52, 38, 44, 93, 36, 57, 90, 98, 58, 9, 50, 72, 87, 1, 69, 55, 46]
\end{itemize}

\textbf{TinyImageNet} This dataset contains 200 classes. In our experiments, we adopted two split settings: 20 split and 40 split.

Each task in the 20-split setting consists of 10 classes. We adopted five random tasks orders as follows:
\begin{itemize}
	\item seed 0: [10, 18, 16, 14, 0, 17, 11, 2, 3, 9, 5, 7, 4, 19, 6, 15, 8, 1, 13, 12]
	\item seed 1: [11, 5, 17, 19, 9, 0, 16, 1, 15, 6, 10, 13, 14, 12, 7, 3, 8, 2, 18, 4]
	\item seed 2: [7, 6, 17, 8, 19, 15, 13, 0, 3, 9, 14, 4, 10, 12, 16, 5, 11, 18, 2, 1]
	\item seed 3: [8, 3, 6, 5, 15, 16, 2, 12, 0, 1, 13, 10, 19, 9, 14, 11, 4, 17, 18, 7]
	\item seed 4: [17, 19, 10, 14, 5, 18, 16, 11, 4, 8, 6, 0, 13, 1, 2, 15, 12, 3, 9, 7]
\end{itemize}

Each task in the 40-split setting consists of 5 classes. We adopted five random tasks orders as follows:
\begin{itemize}
	\item seed 0: [22, 20, 25, 4, 10, 15, 28, 11, 18, 29, 27, 35, 37, 2, 39, 30, 34, 16, 36, 8, 13, 5, 17, 14, 33, 7, 32, 1, 26, 12, 31, 24, 6, 23, 21, 19, 9, 38, 3, 0]
	\item seed 1: [2, 31, 3, 21, 27, 29, 22, 39, 19, 26, 32, 17, 30, 36, 33, 28, 4, 14, 10, 35, 23, 24, 34, 20, 18, 25, 6, 13, 7, 38, 1, 16, 0, 15, 5, 11, 9, 8, 12, 37]
	\item seed 2: [27, 9, 14, 0, 2, 30, 13, 36, 17, 37, 38, 29, 24, 12, 16, 1, 33, 23, 25, 19, 32, 10, 4, 6, 3, 34, 5, 28, 20, 26, 39, 21, 35, 31, 7, 11, 18, 22, 8, 15]
	\item seed 3: [29, 16, 9, 27, 4, 18, 28, 38, 15, 26, 25, 11, 30, 32, 13, 34, 39, 37, 5, 1, 31, 2, 22, 17, 14, 7, 12, 20, 36, 6, 23, 35, 33, 10, 19, 21, 0, 8, 3, 24]
	\item seed 4: [28, 39, 4, 15, 26, 20, 31, 7, 16, 11, 19, 33, 12, 18, 38, 13, 10, 22, 32, 25, 17, 36, 29, 14, 2, 24, 27, 6, 35, 34, 21, 37, 0, 3, 30, 9, 8, 23, 1, 5]
\end{itemize}

\begin{table*}[ht]
	\centering
	\resizebox{0.8\linewidth}{!}{
		\begin{tabular}{@{}l|c|c|c|cc|c|cc|cc@{}}
			\toprule
			Dataset    & CIFAR100 Split & CIFAR100 Super & PMNIST & \multicolumn{2}{c|}{Five dataset} & Omniglot & \multicolumn{2}{c|}{TinyImageNet 40-split} & \multicolumn{2}{c}{TinyImageNet 20-split} \\
			\midrule
			Network & LeNet & LeNet &  LeNet    & AlexNet & ResNet18    & extended LeNet   &  AlexNet & ResNet18 &  AlexNet & ResNet18   \\
			\midrule
			$n$    & 20 & 20 & 12     & 12 & 12    & 20     & 50  & 50  & 50  & 50 \\
			$n_r$  & 1  & 1  & 1      & 1 &  3     & 1     & 20 & 25  & 25 & 25 \\
			$\alpha$ & 0.9  & 0.9 & 0.9    & 0.9 & 0.95      & 0.99    & 0.9 & 0.9 & 0.9 & 0.9   \\
			LR       & 1e-3 & 1e-3 & 1e-3  & 1e-3 & 1e-3    & 5e-3    & 1e-3  & 5e-4  & 1e-3 & 5e-4     \\
			$\lambda_0$ & 1e-4 & 1e-4 & 1e-6     & 1e-6 & 1e-6    & 9e-5     & 5e-4 & 5e-4 & 1e-5 & 5e-4  \\
			$\lambda_1$ & 1e-4 & 1e-4 & 1e-3     & 1e-4 & 1e-4     & 1e-4    & 1e-4 & 1e-4 & 5e-4 & 1e-4   \\
			Bcsz        & 128 & 128 & 128    & 128 & 128      & 16    & 32 & 32  & 32 & 32      \\ \bottomrule
	\end{tabular}}
	\caption{Hyperparameters for the experiments. $n$: total epoch. $n_r$: warm-up epochs for a new task. LR: Learning rate. $\lambda_0,\lambda_1$: coefficients for the regularization terms as discussed in Appendix~\ref{Sect:discussionRegularization}. Bcsz: training batch size.}
	\label{tab:hyperparameters}
\end{table*}

\subsection{Model Architecture}
\label{Sect:app-archi}
In the experimental evaluations, we implement various kinds of backbone architectures to demonstrate our perturbation method for different deep models. We introduce the model architectures used in the paper 

\textbf{LeNet}
We implement two kinds of LeNet models: \textbf{1):} For CIFAR100 Splits/SuperClass and P-MNIST, we implement the general LeNet model with neurons 20-20-50-800-500. \textbf{2):} For Omniglot-Rotation, we follow~\cite{APD:YoonKYH20,kang2022forget} to implement the enlarged LeNet model with neurons 64-128-2500-1500.

\textbf{AlexNet}
For the experiments on Five-dataset and TinyImageNet, we implement AlexNet model by following~\cite{saha2021gradient,kang2022forget}.

\textbf{ResNet-18}
We adopted the reduced ResNet-18 model (i.e., reduce half of the filters in each convolutional layer from the standard Resnet18) on the Five-dataset and TinyImageNet by following~\cite{kang2022forget}.

\subsection{Training Hyperparameters}\label{Sect:hyperparameters}
We reimplement the baselines by rigorously following the official code release or publicly accessible implementations and tested our proposed algorithm with a unified test-bed with the same hyperparameters to get fair comparison results. The training details for our experiments are illustrated in Table~\ref{tab:hyperparameters}. The hyperparameters are selected via grid search. We also provide descriptions of the hyperparameters in the source code.

\end{document}